\documentclass[10pt,twocolumn,letterpaper]{article}

\usepackage{cvpr}
\usepackage{times}
\usepackage{epsfig}
\usepackage{graphicx}
\usepackage{amsmath}
\usepackage{amssymb}

\usepackage{amsthm}
\usepackage{bm}
\usepackage[T1]{fontenc}
\newcommand{\argmax}{\mathop{\rm argmax}\limits}
\newcommand{\argmin}{\mathop{\rm argmin}\limits}
\newtheorem{theorem}{Theorem}
\usepackage{multirow}
\usepackage{capt-of}
\usepackage{pifont}

\newcommand{\bhline}[1]{\noalign{\hrule height #1}}

\usepackage[pagebackref=true,breaklinks=true,letterpaper=true,colorlinks,bookmarks=false]{hyperref}

\cvprfinalcopy

\def\cvprPaperID{5720}

\begin{document}

\title{Label-Noise Robust Generative Adversarial Networks \\ \vspace{-2mm}}

\author{
  Takuhiro Kaneko$^1$
  \quad Yoshitaka Ushiku$^1$
  \quad Tatsuya Harada$^{1,2}$
  \vspace{3mm}\\
  $^1$The University of Tokyo
  \quad $^2$RIKEN
  \vspace{-1mm}
}

\twocolumn[{
  \renewcommand\twocolumn[1][]{#1}
  \maketitle
  \begin{center}
    \includegraphics[width=\textwidth]{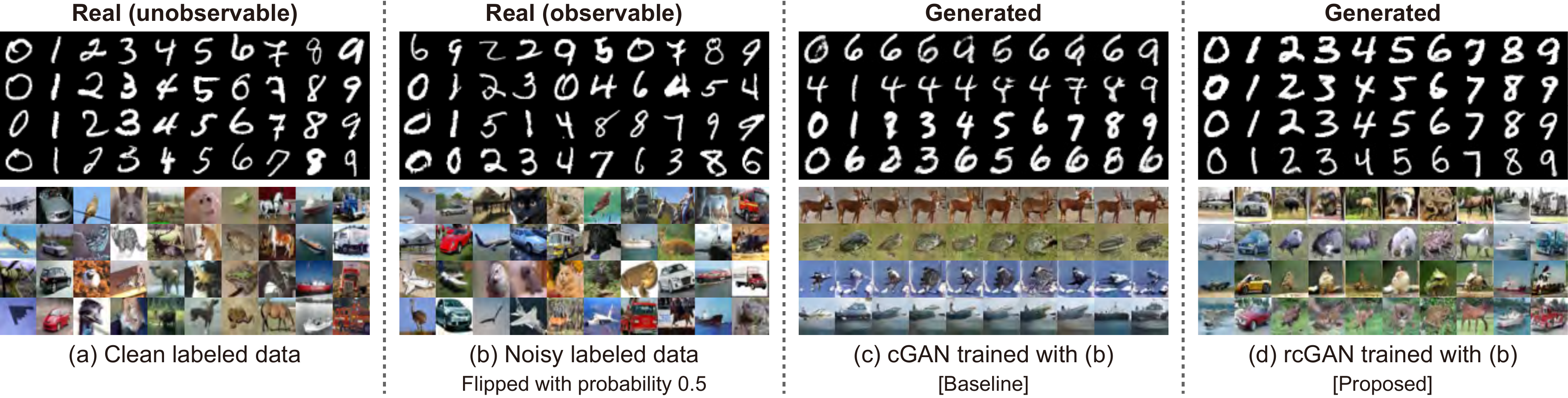}
  \end{center}
  \vspace{-3mm}
  \captionof{figure}{
    Examples of label-noise robust conditional image generation. Each column shows samples belonging to the same class. In (c) and (d), each row contains samples generated with a fixed ${\bm z}$ and a varied $y^g$. Our goal is, given \textit{noisy} labeled data (b), to learn a conditional generative distribution that corresponds with \textit{clean} labeled data (a). When naive cGAN (c) is trained with (b), it fails to learn the disentangled representations, disturbed by \textit{noisy} labeled data. In contrast, proposed rcGAN (d) succeeds in learning the representations disentangled on the basis of \textit{clean} labels, which are close to (a), even when we can only access the \textit{noisy} labeled data (b) during training.}
  \label{fig:concept}
  \vspace{5mm}
}]

\begin{abstract}
  Generative adversarial networks (GANs) are a framework that learns a generative distribution through adversarial training. Recently, their class-conditional extensions (e.g., conditional GAN (cGAN) and auxiliary classifier GAN (AC-GAN)) have attracted much attention owing to their ability to learn the disentangled representations and to improve the training stability. However, their training requires the availability of large-scale accurate class-labeled data, which are often laborious or impractical to collect in a real-world scenario. To remedy this, we propose a novel family of GANs called label-noise robust GANs (rGANs), which, by incorporating a noise transition model, can learn a clean label conditional generative distribution even when training labels are noisy. In particular, we propose two variants: rAC-GAN, which is a bridging model between AC-GAN and the label-noise robust classification model, and rcGAN, which is an extension of cGAN and solves this problem with no reliance on any classifier. In addition to providing the theoretical background, we demonstrate the effectiveness of our models through extensive experiments using diverse GAN configurations, various noise settings, and multiple evaluation metrics (in which we tested 402 conditions in total). Our code is available at \url{https://github.com/takuhirok/rGAN/}.
  \vspace{-5mm}
\end{abstract}

\newpage
\section{Introduction}
\label{sec:introduction}

In computer vision and machine learning, generative modeling has been actively studied to generate or reproduce samples indistinguishable from real data. Recently, deep generative models have emerged as a powerful framework for addressing this problem. Among them, generative adversarial networks (GANs)~\cite{IGoodfellowNIPS2014}, which learn a generative distribution through adversarial training, have become a prominent one owing to their ability to learn any data distribution without explicit density estimation. This mitigates oversmoothing resulting from data distribution approximation, and GANs have succeeded in producing high-fidelity data for various tasks \cite{TKarrasICLR2018,TMiyatoICLR2018b,HZhangArXiv2018,ABrockArXiv2018,PIsolaCVPR2017,CLedigCVPR2017,AShrivastavaCVPR2017,TKimICML2017,ZYiICCV2017,JYZhuICCV2017,SIizukaTOG2017,MLiuNIPS2017,YChoiCVPR2018,TCWangCVPR2018,TCWangArXiv2018,JZhuECCV2016,ABrockICLR2017,TKanekoCVPR2017}.

Along with this success, various extensions of GANs have been proposed. Among them, class-conditional extensions (e.g., conditional GAN (cGAN)~\cite{MMirzaArXiv2014,TMiyatoICLR2018} and auxiliary classifier GAN (AC-GAN)~\cite{AOdenaICML2017}) have attracted much attention mainly for two reasons. (1) By incorporating class labels as supervision, they can learn the representations that are disentangled between the class labels and the other factors. This allows them to selectively generate images conditioned on the class labels~\cite{MMirzaArXiv2014,AOdenaICML2017,TKanekoCVPR2017,ZZhangCVPR2017,TKanekoCVPR2018,YChoiCVPR2018}. Recently, this usefulness has also been demonstrated in class-specific data augmentation~\cite{MFAdarISBI2018,ZZhangCVPR2018}. (2) The added supervision simplifies the learned target from an overall distribution to the conditional distribution. This helps stabilize the GAN training, which is typically unstable, and improves image quality~\cite{AOdenaICML2017,TMiyatoICLR2018,HZhangArXiv2018,ABrockArXiv2018}.

In contrast to these powerful properties, a possible limitation is that typical models rely on the availability of large-scale accurate class-labeled data and their performance depends on their accuracy. Indeed, as shown in Figure~\ref{fig:concept}(c), when conventional cGAN is applied to noisy labeled data (where half labels are randomly flipped, as shown in Figure~\ref{fig:concept}(b)), its performance is significantly degraded, influenced by the noisy labels. When datasets are constructed in real-world scenarios (e.g., crawled from websites or annotated via crowdsourcing), they tend to contain many mislabeled data (e.g., in Clothing1M~\cite{TXiaoCVPR2015}, the overall annotation accuracy is only 61.54\%). Therefore, this limitation would restrict application.

Motivated by these backgrounds, we address the following problem: \textit{``How can we learn a clean label conditional distribution even when training labels are noisy?''} To solve this problem, we propose a novel family of GANs called \textit{label-noise robust GANs} (\textit{rGANs}) that incorporate a noise transition model representing a transition probability between the clean and noisy labels. In particular, we propose two variants: \textit{rAC-GAN}, which is a bridging model between AC-GAN~\cite{AOdenaICML2017} and the label-noise robust classification model, and \textit{rcGAN}, which is an extension of cGAN~\cite{MMirzaArXiv2014,TMiyatoICLR2018} and solves this problem with no reliance on any classifier. As examples, we show generated image samples using rcGAN in Figure~\ref{fig:concept}(d). As shown in this figure, our rcGAN is able to generate images conditioned on clean labels even where conventional cGAN suffers from severe degradation.

Another important issue regarding learning deep neural networks (DNNs) using noisy labeled data is the memorization effect. In image classification, a recent study~\cite{CZhangICLR2017} empirically demonstrated that DNNs can fit even noisy (or random) labels. Another study~\cite{DArpitICML2017} experimentally showed that there are qualitative differences between DNNs trained on clean and noisy labeled data. To the best of our knowledge, no previous studies have sufficiently examined such an effect for conditional deep generative models. Motivated by these facts, in addition to providing a theoretical background on rAC-GAN and rcGAN, we conducted extensive experiments to examine the gap between theory and practice. In particular, we evaluated our models using diverse GAN configurations from standard to state-of-the-art in various label-noise settings including synthetic and real-world noise. We also tested our methods in the case when a noise transition model is known and in the case when it is not. Furthermore, we introduce an improved technique to stabilize training in a severely noisy setting (e.g., that in which 90\% of the labels are corrupted) and show the effectiveness.

Overall, our contributions are summarized as follows:
\begin{itemize}
  \vspace{-1mm}
  \setlength{\parskip}{1pt}
  \setlength{\itemsep}{0pt}
\item We tackle a novel problem called \textit{label-noise robust conditional image generation}, in which the goal is to learn a clean label conditional generative distribution even when training labels are noisy.
\item To solve this problem, we propose a new family of GANs called \textit{rGANs} that incorporate a noise transition model into conditional extensions of GANs. In particular, we propose two variants, i.e., \textit{rAC-GAN} and \textit{rcGAN}, for the two representative class-conditional GANs, i.e., AC-GAN and cGAN.
\item In addition to providing a theoretical background, we examine the gap between theory and practice through extensive experiments (in which we tested 402 conditions in total). Our code is available at \url{https://github.com/takuhirok/rGAN/}.
  \vspace{-1mm}
\end{itemize}

\section{Related work}
\label{sec:related}

\noindent\textbf{Deep generative models.}
Generative modeling has been a fundamental problem and has been actively studied in computer vision and machine learning. Recently, deep generative models have emerged as a powerful framework. Among them, three popular approaches are GANs~\cite{IGoodfellowNIPS2014}, variational autoencoders (VAEs)~\cite{DKingmaICLR2014,DRezendeICML2014}, and autoregressive models (ARs)~\cite{AOordICML2016}. All these models have pros and cons. One well-known problem with GANs is training instability; however, the recent studies have been making a great stride in solving this problem~\cite{EDentonNIPS2015,ARadfordICLR2016,TSalimansNIPS2016,JZhaoICLR2017,MArjovskyICLR2017,MArjovskyICML2017,XMaoICCV2017,IGulrajaniNIPS2017,TKarrasICLR2018,XWeiICLR2018,TMiyatoICLR2018b,LMeschederICML2018,HZhangArXiv2018,ABrockArXiv2018}. In this paper, we focus on GANs because they have flexibility to the data representation, allowing for incorporating a noise transition model. However, with regard to VAEs and ARs, conditional extensions~\cite{DKingmaNIPS2014,EMansimovICLR2016,XYanECCV2016,AOordNIPS2016,SReedICLRW2017} have been proposed, and incorporating our ideas into them is a possible direction of future work.

\smallskip\noindent\textbf{Conditional extensions of GANs.}
As discussed in Section~\ref{sec:introduction}, conditional extensions of GANs have been actively studied to learn the representations that are disentangled between the conditional information and the other factors or to stabilize training and boost image quality. Other than class or attribute labels~\cite{MMirzaArXiv2014,AOdenaICML2017,TKanekoCVPR2017,ZZhangCVPR2017,TKanekoCVPR2018,YChoiCVPR2018}, texts~\cite{SReedICML2016,HZhangICCV2017,HZhangArXiv2017,XTaoCVPR2018}, object locations~\cite{SReedNIPS2016}, images~\cite{EDentonNIPS2015,PIsolaCVPR2017,CLedigCVPR2017,TCWangCVPR2018}, or videos~\cite{TCWangArXiv2018} are used as conditional information, and the effectiveness of conditional extensions of GANs has also been verified for them. In this paper, we focus on the situation in which noise exists in the label domain because obtaining robustness in such a domain has been a fundamental and important problem in image classification and has been actively studied, as discussed in the next paragraph. However, also in other domains (e.g., texts or images), it is highly likely that noise may exist when data are collected in real-world scenarios (e.g., crawled from websites or annotated via crowdsourcing). We believe that our findings would help the research also in these domains.

\smallskip\noindent\textbf{Label-noise robust models.}
Learning with noisy labels has been keenly studied since addressed in the learning theory community~\cite{DAngluinML1988,NNatarajanNIPS2013}. Lately, this problem has also been studied in image classification with DNNs. For instance, to obtain label-noise robustness, one approach replaces a typical cross-entropy loss with a noise-tolerant loss~\cite{AGhoshAAAI2017,ZZhangArXiv2018}. Another approach cleans up labels or selects clean labels out of noisy labels using neural network predictions or gradient directions~\cite{SReedICLR2015,DTanakaCVPR2018,EMalachNIPS2017,LJiangICML2018,MRenICML2018,BHanArXiv2018}. The other approach incorporates a noise transition model~\cite{SSukhbaatarICLRW2015,IJindalICDM2016,GPatriniCVPR2017,JGoldbergerICLR2017}, similarly to ours. These studies show promising results in both theory and practice and our study is based on their findings.

The main difference from them is that their goal is to obtain label-noise robustness in image classification, but our goal is to obtain such robustness in conditional image generation. We remark that our developed rAC-GAN internally uses a classifier; thus, it can be viewed as a bridging model between noise robust image classification and conditional image generation. Note that we also developed rcGAN, which is a classifier-free model, motivated by the recent studies~\cite{AOdenaICML2017,TMiyatoICLR2018} that indicate that AC-GAN tends to lose diversity through a side effect of generating recognizable (i.e., classifiable) images. Another related topic is \textit{pixel}-noise robust image generation~\cite{ABoraICLR2018,JLehtinenICML2018}. The difference from them is that they focused on the noise inserted in a \textit{pixel} domain, but we focus on the noise in a \textit{label} domain.

\section{Notation and problem statement}
\label{sec:notation}
We begin by defining notation and the problem statement. Throughout, we use superscript $r$ to denote the real distribution and $g$ the generative distribution. Let ${\bm x} \in {\cal X}$ be the target data (e.g., images) and $y \in {\cal Y}$ the corresponding class label. Here, ${\cal X}$ is the data space ${\cal X} \subseteq {\mathbb R}^d$, where $d$ is the dimension of the data, and ${\cal Y}$ is the label space ${\cal Y} = \{1, \dots, c \}$, where $c$ is the number of classes. We assume that $y$ is noisy (and we denote such noisy label by $\tilde{y}$) and there exists a corresponding clean label $\hat{y}$ that we cannot observe during training. In particular, we assume \textit{class-dependent} noise in which each clean label $\hat{y} = i$ is corrupted to a noisy label $\tilde{y} = j$ with a probability $p(\tilde{y} = j |\hat{y} = i) = T_{i, j}$, independently of ${\bm x}$, where we define a noise transition matrix as $T = (T_{i, j})\in [0, 1]^{c \times c}$ $(\sum_i T_{i, j} = 1)$. Note that this assumption is commonly used in label-noise robust image classification (e.g., \cite{AGhoshAAAI2017,ZZhangArXiv2018,SSukhbaatarICLRW2015,IJindalICDM2016,GPatriniCVPR2017,JGoldbergerICLR2017}).

Our task is, when given noisy labeled samples $({\bm x}^r, \tilde{y}^r) \sim \tilde{p}^r({\bm x}, \tilde{y})$, to construct a label-noise robust conditional generator such that $\hat{p}^g({\bm x}, \hat{y}) = \hat{p}^r({\bm x}, \hat{y})$, which can generate ${\bm x}$ conditioned on \textit{clean} $\hat{y}$ rather than conditioned on \textit{noisy} $\tilde{y}$. This task is challenging for typical conditional generative models, such as AC-GAN~\cite{AOdenaICML2017} (Figure~\ref{fig:networks}(b)) and cGAN~\cite{MMirzaArXiv2014,TMiyatoICLR2018} (Figure~\ref{fig:networks}(d)), because they attempt to construct a generator conditioned on the observable labels; i.e., in this case, they attempt to construct a \textit{noisy}-label-dependent generator that generates ${\bm x}$ conditioned on \textit{noisy} $\tilde{y}$ rather than conditioned on \textit{clean} $\hat{y}$. Our main idea for solving this problem is to incorporate a noise transition model, i.e., $p(\tilde{y}|\hat{y})$, into these models (viewed as orange rectangles in Figures~\ref{fig:networks}(c) and (e)). In particular, we develop two variants: rAC-GAN and rcGAN. We describe their details in Sections~\ref{sec:rAC-GAN} and \ref{sec:rcGAN},
respectively.

\begin{figure*}[t]
  \centering
  \includegraphics[width=1.0\textwidth]{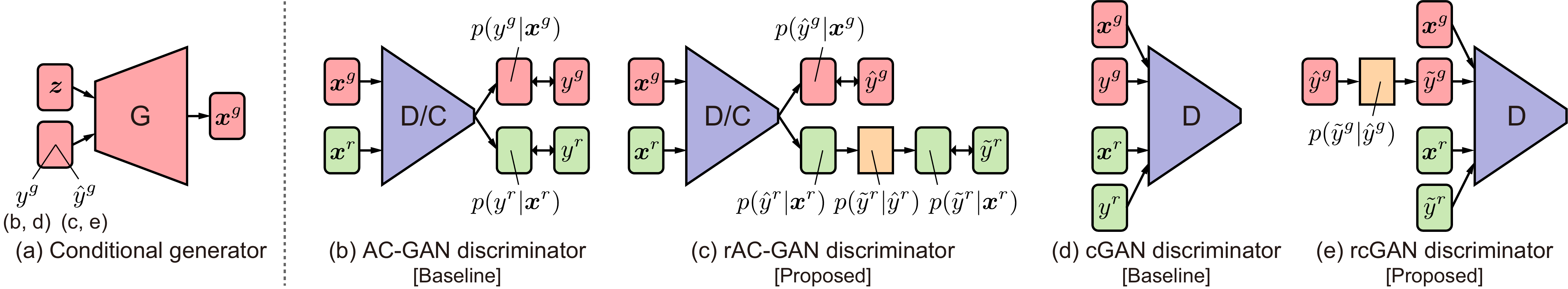}
  \caption{Comparison of naive and label-noise robust GANs. We denote the generator, discriminator, and auxiliary classifier by $G$, $D$, and $C$, respectively. Among all models, conditional generators (a) are similar. In our rAC-GAN (c) and rcGAN (e), we incorporate a noise transition model (viewed as an orange rectangle) into AC-GAN (b) and cGAN (d), respectively.}
  \label{fig:networks}
  \vspace{-3mm}
\end{figure*}

\section{Label-noise robust AC-GAN: rAC-GAN}
\label{sec:rAC-GAN}

\subsection{Background: AC-GAN}
\label{subsec:AC-GAN}

AC-GAN~\cite{AOdenaICML2017} is one of representative conditional extensions of GANs~\cite{IGoodfellowNIPS2014}. AC-GAN learns a conditional generator $G$ that transforms noise ${\bm z}$ and label ${y}^g$ into data ${\bm x}^g = G({\bm z}, {y}^g)$ with two networks. One is a discriminator $D$ that assigns probability $p = D({\bm x})$ for samples ${\bm x} \sim p^r({\bm x})$ and assigns $1 - p$ for samples ${\bm x} \sim p^g({\bm x})$. The other is an auxiliary classifier $C(y|{\bm x})$ that represents a probability distribution over class labels given ${\bm x}$. These networks are optimized by using two losses, namely, an adversarial loss and an auxiliary classifier loss.

\smallskip\noindent\textbf{Adversarial loss.}
An adversarial loss is defined as
\begin{flalign}
  \label{eqn:gan}
  {\cal L}_{\rm{GAN}} = & \: \mathbb{E}_{{\bm x}^r \sim p^r({\bm x})} [\log D({\bm x}^r)]
  \nonumber \\
  + & \: \mathbb{E}_{{\bm z} \sim p({\bm z}), {y}^g \sim p(y)} [\log (1 - D(G({\bm z}, {y}^g)))],
\end{flalign}
where $D$ attempts to find the best decision boundary between real and generated data by maximizing this loss, and $G$ attempts to generate data indistinguishable by $D$ by minimizing this loss.

\smallskip\noindent\textbf{Auxiliary classifier loss.}
An auxiliary classifier loss is used to make the generated data belong to the target class. To achieve this, first $C$ is optimized using a classification loss of real data:
\begin{flalign}
  \label{eqn:acgan_cls_r}
  {\cal L}_{\rm{AC}}^r = \mathbb{E}_{({\bm x}^r, {y}^r) \sim p^r ({\bm x}, y)} [-\log C(y = {y}^r|{\bm x}^r)],
\end{flalign}
where $C$ learns to classify real data to the corresponding class by minimizing this loss. Then, $G$ is optimized by using a classification loss of generated data:
\begin{flalign}
  \label{eqn:acgan_cls_g}
  {\cal L}_{\rm{AC}}^g = \mathbb{E}_{{\bm z} \sim p({\bm z}), {y}^g \sim p(y)} [-\log C(y = {y}^g|G({\bm z}, {y}^g))],
\end{flalign}
where $G$ attempts to generate data belonging to the corresponding class by minimizing this loss.

\smallskip\noindent\textbf{Full objective.}
In practice, shared networks between $D$ and $C$ are commonly used~\cite{AOdenaICML2017,IGulrajaniNIPS2017}. In this setting, the full objective is written as
\begin{flalign}
  \label{eqn:acgan}
  {\cal L}_{D\mbox{/}C} = & \: -{\cal L}_{\rm{GAN}} + \lambda_{\rm AC}^r {\cal L}_{\rm{AC}}^r,
  \\
  {\cal L}_{G} = & \: {\cal L}_{\rm{GAN}} + \lambda_{\rm AC}^g {\cal L}_{\rm{AC}}^g,
\end{flalign}
where $\lambda_{\rm AC}^r$ and $\lambda_{\rm AC}^g$ are the trade-off parameters between the adversarial loss and the auxiliary classifier loss for the real and generated data, respectively. $D\mbox{/}C$ and $G$ are optimized by minimizing ${\cal L}_{D\mbox{/}C}$ and ${\cal L}_{G}$, respectively.

\subsection{rAC-GAN}
\label{subsec:rAC-GAN}

By the above definition, when ${y}^r$ is noisy (i.e., $\tilde{y}^r$ is given) and $C$ fits such noisy labels,\footnote{Zhang et al.~\cite{CZhangICLR2017} discuss generalization and memorization of DNNs and empirically demonstrated that DNNs are capable of fitting even noisy (or random) labels. Although other studies empirically demonstrated that some techniques (e.g., dropout~\cite{DArpitICML2017}, mixup~\cite{HZhangICLR2018}, and high learning rate~\cite{DTanakaCVPR2018}) are useful for preventing DNNs from memorizing noisy labels, their theoretical support still remains as an open issue. In this paper, we conducted experiments on various GAN configurations to investigate such effect in our task. See Section~\ref{subsec:comp_eval} for details.} AC-GAN learns the \textit{noisy} label conditional generator $G({\bm z},  \tilde{y}^g)$. In contrast, our goal is to construct the \textit{clean} label conditional generator $G({\bm z}, \hat{{y}}^g)$. To achieve this goal, we incorporate a noise transition model (i.e., $p(\tilde{y}|\hat{y})$; viewed as an orange rectangle in Figure~\ref{fig:networks}(c)) into the auxiliary classifier. In particular, we reformulate the auxiliary classifier loss as
\begin{flalign}
  \label{eqn:racgan_cls_r}
  & \: {\cal L}_{\rm{rAC}}^r = \mathbb{E}_{({\bm x}^r, \tilde{y}^r) \sim \tilde{p}^r ({\bm x}, \tilde{y})} [-\log \tilde{C}(\tilde{y}=\tilde{y}^r|{\bm x}^r)]
  \nonumber \\
  = & \: \mathbb{E}_{({\bm x}^r, \tilde{y}^r) \sim \tilde{p}^r ({\bm x}, \tilde{y})} \nonumber \\
  & \:\:\:\:\:\:\:\:\:\:\:\:
  [-\log \sum_{\hat{y}^r} p(\tilde{y} = \tilde{y}^r|\hat{y} = \hat{y}^r) \hat{C}(\hat{y} = \hat{y}^r|{\bm x}^r)]
  \nonumber \\
  = & \: \mathbb{E}_{({\bm x}^r, \tilde{y}^r) \sim \tilde{p}^r ({\bm x}, \tilde{y})} [-\log \sum_{\hat{y}^r} T_{\hat{y}^r, \tilde{y}^r} \hat{C}(\hat{y} = \hat{y}^r|{\bm x}^r)],  
\end{flalign}
where we denote the \textit{noisy} label classifier by $\tilde{C}$ and the \textit{clean} label classifier by $\hat{C}$ (and we explain the reason why we call it \textit{clean} in Theorem~\ref{th:racgan}). Between the first and second lines, we assume that the noise transition is independent of ${\bm x}$, as mentioned in Section~\ref{sec:notation}. Note that this formulation (called the \textit{forward correction}) is often used in label-noise robust classification models~\cite{SSukhbaatarICLRW2015,IJindalICDM2016,GPatriniCVPR2017,JGoldbergerICLR2017} and rAC-GAN can be viewed as a bridging model between GANs and them. In naive AC-GAN, $\tilde{C}$ is optimized for ${\cal L}_{\rm{AC}}^r$, whereas in our rAC-GAN, $\hat{C}$ is optimized for ${\cal L}_{\rm{rAC}}^r$. Similarly, $G$ is optimized using $\hat{C}$ rather than using $\tilde{C}$:
\begin{eqnarray}
  \label{eqn:racgan_cls_g}
  {\cal L}_{\rm{rAC}}^g = \mathbb{E}_{{\bm z} \sim p({\bm z}), \hat{y}^g \sim p(\hat{y})} [-\log \hat{C}(\hat{y} = \hat{y}^g|G({\bm z}, \hat{y}^g))].
\end{eqnarray}

\smallskip\noindent\textbf{Theoretical background.}
In the above, we use a cross-entropy loss, which is a kind of proper composite loss~\cite{MReidJMLR2010}. In this case, Theorem 2 in~\cite{GPatriniCVPR2017} shows that minimizing the forward corrected loss (i.e., Equation~\ref{eqn:racgan_cls_r}) is equal to minimizing the original loss under the clean distribution. More precisely, the following theorem holds.
\begin{theorem}
  \label{th:racgan}
  When $T$ is nonsingular,
  \begin{flalign}
    \label{eqn:forward}
    & \: \argmin_{\hat{C}} \mathbb{E}_{({\bm x}^r, \tilde{y}^r) \sim \tilde{p}^r ({\bm x}, \tilde{y})} [-\log \sum_{\hat{y}^r} T_{\hat{y}^r, \tilde{y}^r} \hat{C}(\hat{y} = \hat{y}^r|{\bm x}^r)]
    \nonumber \\
    = & \: \argmin_{\hat{C}} \mathbb{E}_{({\bm x}^r, \hat{y}^r) \sim \hat{p}^r ({\bm x}, \hat{y})} [-\log \hat{C}(\hat{y} = \hat{y}^r|{\bm x}^r)].
  \end{flalign}
\end{theorem}
For a detailed proof, refer to Theorem 2 in~\cite{GPatriniCVPR2017}. This supports the idea that, by minimizing ${\cal L}_{\rm{rAC}}^r$ for noisy labeled samples, we can obtain $\hat{C}$ that classifies ${\bm x}$ as its corresponding clean label $\hat{y}$. In rAC-GAN, $G$ is optimized for this \textit{clean} classifier $\hat{C}$; hence, in $G$'s input space, $\hat{y}^g$ is encouraged to represent clean labels.

\section{Label-noise robust cGAN: rcGAN}
\label{sec:rcGAN}

\subsection{Background: cGAN}
\label{subsec:cGAN}

cGAN~\cite{MMirzaArXiv2014,TMiyatoICLR2018} is another representative conditional extension of GANs~\cite{IGoodfellowNIPS2014}. In cGAN, a conditional generator $G({\bm z}, {y}^g)$ and a conditional discriminator $D({\bm x}, y)$ are jointly trained using a conditional adversarial loss.

\smallskip\noindent\textbf{Conditional adversarial loss.}
A conditional adversarial loss is defined as
\begin{flalign}
  \label{eqn:cgan}
  & {\cal L}_{\rm{cGAN}} = \mathbb{E}_{({\bm x}^r, {y}^r) \sim p^r({\bm x}, {y})} [\log D({\bm x}^r, {y}^r)]
  \nonumber \\
  & \:\:\:\:\:\:\:\: + \mathbb{E}_{{\bm z} \sim p({\bm z}), {y}^g \sim p({y})} [\log (1 - D(G({\bm z}, {y}^g), {y}^g))],
\end{flalign}
where $D$ attempts to find the best decision boundary between real and generated data conditioned on $y$ by maximizing this loss. In contrast, $G$ attempts to generate data indistinguishable by $D$ with a constraint on ${y}^g$ by minimizing this loss. In an optimal condition~\cite{IGoodfellowNIPS2014}, cGAN learns $G(\bm{z}, y)$ such that $p^g({\bm x}, {y}) = p^r({\bm x}, {y})$.

\subsection{rcGAN}
\label{subsec:rcGAN}

By the above definition, when ${y}^r$ is noisy (i.e., $\tilde{y}^r$ is given), cGAN learns the \textit{noisy} label conditional generator $G({\bm z}, \tilde{y}^g)$. In contrast, our goal is to construct the \textit{clean} label conditional generator $G({\bm z}, \hat{y}^g)$. To achieve this goal, we insert a noise transition model (viewed as an orange rectangle in Figure~\ref{fig:networks}(e)) before $\hat{y}^g$ is given to $D$. In particular, we sample $\tilde{y}^g$ from $\tilde{y}^g \sim p(\tilde{y} | \hat{y}^g)$ and redefine Equation~\ref{eqn:cgan} as
\begin{flalign}
  \label{eqn:rcgan}
  & {\cal L}_{\rm{rcGAN}} = \mathbb{E}_{({\bm x}^r, \tilde{y}^r) \sim \tilde{p}^r({\bm x}, \tilde{y})} [\log D({\bm x}^r, \tilde{y}^r)]
  \nonumber \\
  & + \mathbb{E}_{{\bm z} \sim p({\bm z}), \hat{y}^g \sim p(\hat{y}), \tilde{y}^g \sim p(\tilde{y} | \hat{y}^g)} [\log (1 - D(G({\bm z}, \hat{y}^g), \tilde{y}^g))],
\end{flalign}
where $D$ attempts to find the best decision boundary between real and generated data conditioned on \textit{noisy} labels $\tilde{y}$, by maximizing this loss. In contrast, $G$ attempts to generate data indistinguishable by $D$ with a constraint on \textit{clean} labels $\hat{y}^g$ (and we explain the rationale behind calling it \textit{clean} in Theorem~\ref{th:rcgan}), by minimizing this loss.

\smallskip\noindent\textbf{Theoretical background.}
In an optimal condition, the following theorem holds.
\begin{theorem}
  \label{th:rcgan}
  When $T$ is nonsingular (i.e., $T$ has a unique inverse), $G$ is optimal if and only if $\hat{p}^g({\bm x}, \hat{y}) = \hat{p}^r({\bm x}, \hat{y})$.
\end{theorem}
\begin{proof}{}
  For $G$ fixed, rcGAN is the same as cGAN where ${y}$ is replaced by $\tilde{y}$. Therefore, by extending Proposition 1 and Theorem 1 in~\cite{IGoodfellowNIPS2014} (GAN optimal solution) to a conditional setting, the optimal discriminator $D$ for fixed $G$ is
  \begin{flalign}
    \label{eqn:optim_d_rcgan}
    D({\bm x}, \tilde{y}) = \frac{\tilde{p}^r({\bm x}, \tilde{y})}{\tilde{p}^r({\bm x}, \tilde{y}) + \tilde{p}^g({\bm x}, \tilde{y})}.
  \end{flalign}
  Then $G$ is optimal if and only if
  \begin{flalign}
    \label{eqn:optim_g_rcgan}
    \tilde{p}^g({\bm x}, \tilde{y}) & = \tilde{p}^r({\bm x}, \tilde{y}).
  \end{flalign}
  As mentioned in Section~\ref{sec:notation}, we assume that label corruption occurs with $p(\tilde{y}|\hat{y})$, i.e., independently of ${\bm x}$. In this case,
  \begin{flalign}
    \label{eqn:transition}
    \tilde{p}({\bm x}, \tilde{y}) & = \tilde{p}(\tilde{y}|{\bm x}) p({\bm x}) = \sum_{\hat{y}} p(\tilde{y}|\hat{y}) \hat{p}(\hat{y}|{\bm x}) p({\bm x})
    \nonumber \\
    & = \sum_{\hat{y}} p(\tilde{y}|\hat{y}) \hat{p}({\bm x}, \hat{y})
    = \sum_{\hat{y}} T_{\hat{y}, \tilde{y}} \hat{p}({\bm x}, \hat{y}).
  \end{flalign}
  Substituting Equation~\ref{eqn:transition} into Equation~\ref{eqn:optim_g_rcgan} gives
  \begin{flalign}
    \label{eqn:optim_g_rcgan2}
    \sum_{\hat{y}} T_{\hat{y}, \tilde{y}} \hat{p}^g({\bm x}, \hat{y}) = \sum_{\hat{y}} T_{\hat{y}, \tilde{y}} \hat{p}^r({\bm x}, \hat{y}).
  \end{flalign}
  By considering the matrix form,
  \begin{flalign}
    \label{eqn:optim_g_rcgan3}
    T^{\top} \hat{P}^g = T^{\top} \hat{P}^r,
  \end{flalign}
  where $\hat{P}^g = [\hat{p}^g ({\bm x}, \hat{y} = 1), \dots,
  \hat{p}^g ({\bm x}, \hat{y} = c)]^{\top}$ and
  $\hat{P}^r = [\hat{p}^r ({\bm x}, \hat{y} = 1), \dots,
  \hat{p}^r ({\bm x}, \hat{y} = c)]^{\top}$.
  When $T$ has an inverse,
  \begin{flalign}
    \label{eqn:optim_g_rcgan4}
    T^{\top} \hat{P}^g = T^{\top} \hat{P}^r
    \Leftrightarrow \hat{P}^g = (T^{\top})^{-1} T^{\top} \hat{P}^r = \hat{P}^r.
  \end{flalign}
  As the corresponding elements in $\hat{P}^g$ and $\hat{P}^r$ are equal,
  $\hat{p}^g ({\bm x}, \hat{y}) = \hat{p}^r ({\bm x}, \hat{y})$.
\end{proof}
This supports the idea that, in an optimal condition, rcGAN learns $G({\bm z}, \hat{y})$ such that $\hat{p}^g({\bm x}, \hat{y}) = \hat{p}^r({\bm x}, \hat{y})$.

\renewcommand{\baselinestretch}{0.97}\selectfont

\section{Advanced techniques for practice}
\label{sec:advance}

\subsection{Noise transition probability estimation}
\label{subsec:estT}

In the above, we assume that $T$ is known, but this assumption may be too strict for real-world applications. However, fortunately, previous studies~\cite{SSukhbaatarICLRW2015,IJindalICDM2016,GPatriniCVPR2017,JGoldbergerICLR2017} have been eagerly tackling this problem and several methods for estimating $T'$ (where we denote the estimated $T$ by $T'$) have been proposed. Among them, we tested a \textit{robust two-stage training algorithm}~\cite{GPatriniCVPR2017} in the experiments and analyzed the effects of estimated $T'$. We show the results in Section~\ref{subsec:estT_eval}.

\subsection{Improved technique for severely noisy data}
\label{subsec:imp}

Thorough extensive experiments, we find that some GAN configurations suffer from performance degradation in a severely noisy setting (e.g., in which 90\% of the labels are corrupted). In this type of environment, each label is flipped with a high probability. This disturbs $G$ form associating an image with a label. To strengthen their connection, we incorporate mutual information regularization~\cite{XChenNIPS2016}:
\begin{eqnarray}
  \label{eqn:info}
  {\cal L}_{\rm MI} = \mathbb{E}_{{\bm z} \sim p({\bm z}), \hat{y}^g \sim p(\hat{y})} [-\log Q (\hat{y} = \hat{y}^g | G({\bm z}, \hat{y}^g))],
\end{eqnarray}
where $Q(\hat{y}|{\bm x})$ is an auxiliary distribution approximating a true posterior $p(\hat{y}|{\bm x})$. We optimize $G$ and $Q$ by minimizing this loss with trade-off parameters $\lambda_{\rm MI}^g$ and $\lambda_{\rm MI}^q$, respectively. This formulation is similar to Equation~\ref{eqn:racgan_cls_g}, but the difference is whether $G$ is optimized for $\hat{C}$ (optimized using real images and noisy labels) or for $Q$ (optimized using generated images and clean labels). We demonstrate the effectiveness of this technique in Section~\ref{subsec:imp_eval}.

\section{Experiments}
\label{sec:experiments}

\begin{figure*}[t]
  \centering
  \includegraphics[width=1.0\textwidth]{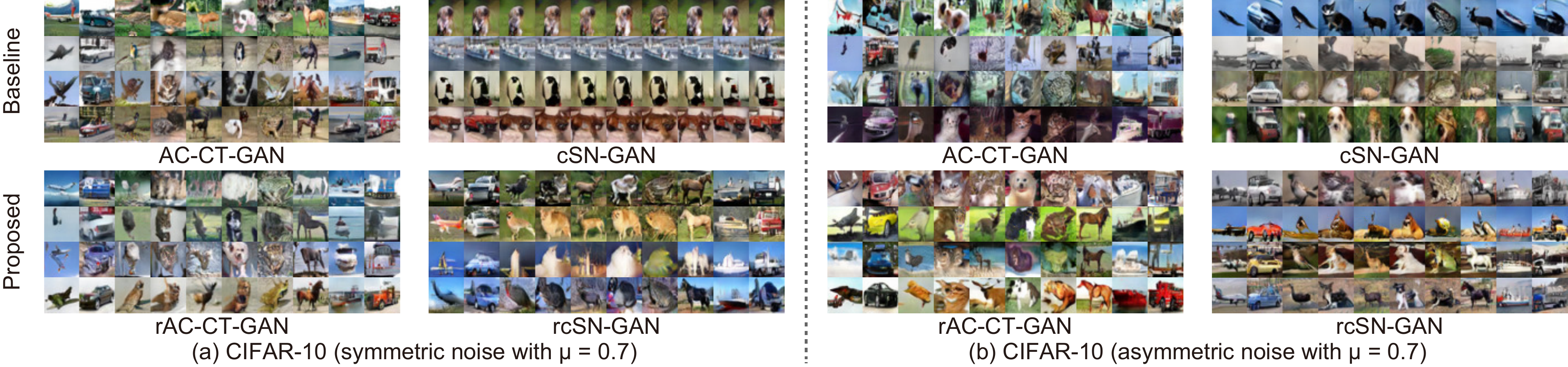}  
  \caption{Generated image samples on CIFAR-10. Each column shows samples belonging to the same class. Each row contains samples generated with a fixed ${\bm z}$ and a varied $y^g$. In symmetric noise (a), cSN-GAN is primarily influenced by noisy labels and fails to learn the disentangled representations. In asymmetric noise (b), it is expected that fourth and sixth columns will include cat and dog, respectively. However, in AC-CT-GAN and cSN-GAN, these columns  contain the inverse. As evidence, we list the accuracy in the fourth column for cat/dog classes in Table~\ref{tab:acc}. These scores indicate that the proposed models are robust but the baselines are weak for the flipped classes. See Figures~\ref{fig:dcgan_gen}--\ref{fig:sngan_gen} in the Appendix for more samples.}
  \label{fig:cifar10_gen}
  \vspace{-5mm}
\end{figure*}

\subsection{Comprehensive study}
\label{subsec:comp_eval}

In Sections~\ref{sec:rAC-GAN} and \ref{sec:rcGAN}, we showed that our approach is theoretically grounded. However, generally, in DNNs, there is still a gap between theory and practice. In particular, the label-noise effect in DNNs just recently began to be discussed in image classification~\cite{CZhangICLR2017,DArpitICML2017}, and it is demonstrated that such a gap exists. However, in conditional image generation, such an effect has not been sufficiently examined. To advance this research, we first conducted a comprehensive study, i.e., compared the performance of conventional AC-GAN and cGAN and proposed rAC-GAN and rcGAN using diverse GAN configurations in various label-noise settings with multiple evaluation metrics.\footnote{Through Sections~\ref{subsec:comp_eval}--\ref{subsec:imp_eval}, we tested 392 conditions in total. For each condition, we trained two models with different initializations and report the results averaged over them.} Due to the space limitation, we briefly review the experimental setup and only provide the important results in this main text. See the Appendix and our \href{https://takuhirok.github.io/rGAN/}{website} for details and more results.

\smallskip\noindent\textbf{Dataset.}
We verified the effectiveness of our method on two benchmark datasets: CIFAR-10 and CIFAR-100~\cite{AKrizhevskyTech2009}, which are commonly used in both image generation and label-noise robust image classification. Both datasets contain $60k$ $32 \times 32$ natural images, which are divided into $50k$ training and $10k$ test images. CIFAR-10 and CIFAR-100 have 10 and 100 classes, respectively. We assumed two label-noise settings that are popularly used in label-noise robust image classification: (1) {\bf Symmetric} (class-independent) noise~\cite{RVanNIPS2015}: For all classes, ground truth labels are replaced with uniform random classes with probability $\mu$. (2) {\bf Asymmetric} (class-dependent) noise~\cite{GPatriniCVPR2017}: Ground truth labels are flipped with probability $\mu$ by mimicking real mistakes between similar classes. Following~\cite{GPatriniCVPR2017}, for CIFAR-10, ground truth labels are replaced with \textit{truck} $\rightarrow$ \textit{automobile}, \textit{bird} $\rightarrow$ \textit{airplane}, \textit{deer} $\rightarrow$ \textit{horse}, and \textit{cat} $\leftrightarrow$ \textit{dog}, and for CIFAR-100, ground truth labels are flipped into the next class circularly within the same superclasses. In both settings, we selected $\mu$ from $\{ 0.1, 0.3, 0.5, 0.7, 0.9\}$.

\smallskip\noindent\textbf{GAN configurations.}
A recent study~\cite{MLucicArXiv2017} shows the sensitivity of GANs to hyperparameters. However, when clean labeled data are not available, it is impractical to tune the hyperparameters for each label-noise setting. Hence, instead of searching for the best model with hyperparameter tuning, we tested various GAN configurations using the default parameters that are typically used in clean label settings and examined the label-noise effect. We chose four models to cover standard, widely accepted, and state-of-the-art models: {\bf DCGAN}~\cite{ARadfordICLR2016}, {\bf WGAN-GP}~\cite{IGulrajaniNIPS2017}, {\bf CT-GAN}~\cite{XWeiICLR2018}, and {\bf SN-GAN}~\cite{TMiyatoICLR2018b}. We implemented AC-GAN, rAC-GAN, cGAN, and rcGAN based on them. For cGAN and rcGAN, we used the \textit{concat} discriminator~\cite{MMirzaArXiv2014} for DCGAN and the \textit{projection} discriminator~\cite{TMiyatoICLR2018} for the others.

\smallskip\noindent\textbf{Evaluation metrics.}
As discussed in previous studies~\cite{LTheisICLR2016,MLucicArXiv2017,KShmelkovECCV2018}, evaluation and comparison of GANs can be challenging partially because of the lack of an explicit likelihood measure. Considering this fact, we used four metrics for a comprehensive analysis: (1) the Fr\'{e}chet Inception distance ({\bf FID}), (2) {\bf Intra FID}, (3) the {\bf GAN-test}, and (4) the {\bf GAN-train}. The FID~\cite{MHeuselNIPS2017} measures the distance between $p^r$ and $p^g$ in Inception embeddings. We used it to assess the quality of an overall generative distribution. Intra FID~\cite{TMiyatoICLR2018} calculates the FID for each class. We used it to assess the quality of a conditional generative distribution.\footnote{We used Intra FID only for CIFAR-10 because, in CIFAR-100, the number of clean labeled data for each class (500) is insufficient.} The GAN-test~\cite{KShmelkovECCV2018} is the accuracy of a classifier trained on real images and evaluated on generated images. This metric approximates the precision (image quality) of GANs. The GAN-train~\cite{KShmelkovECCV2018} is the accuracy of a classifier trained on generated images and evaluated on real images in a test. This metric approximates the recall (diversity) of GANs.

\begin{figure*}[t]
  \centering
  \includegraphics[width=0.95\textwidth]{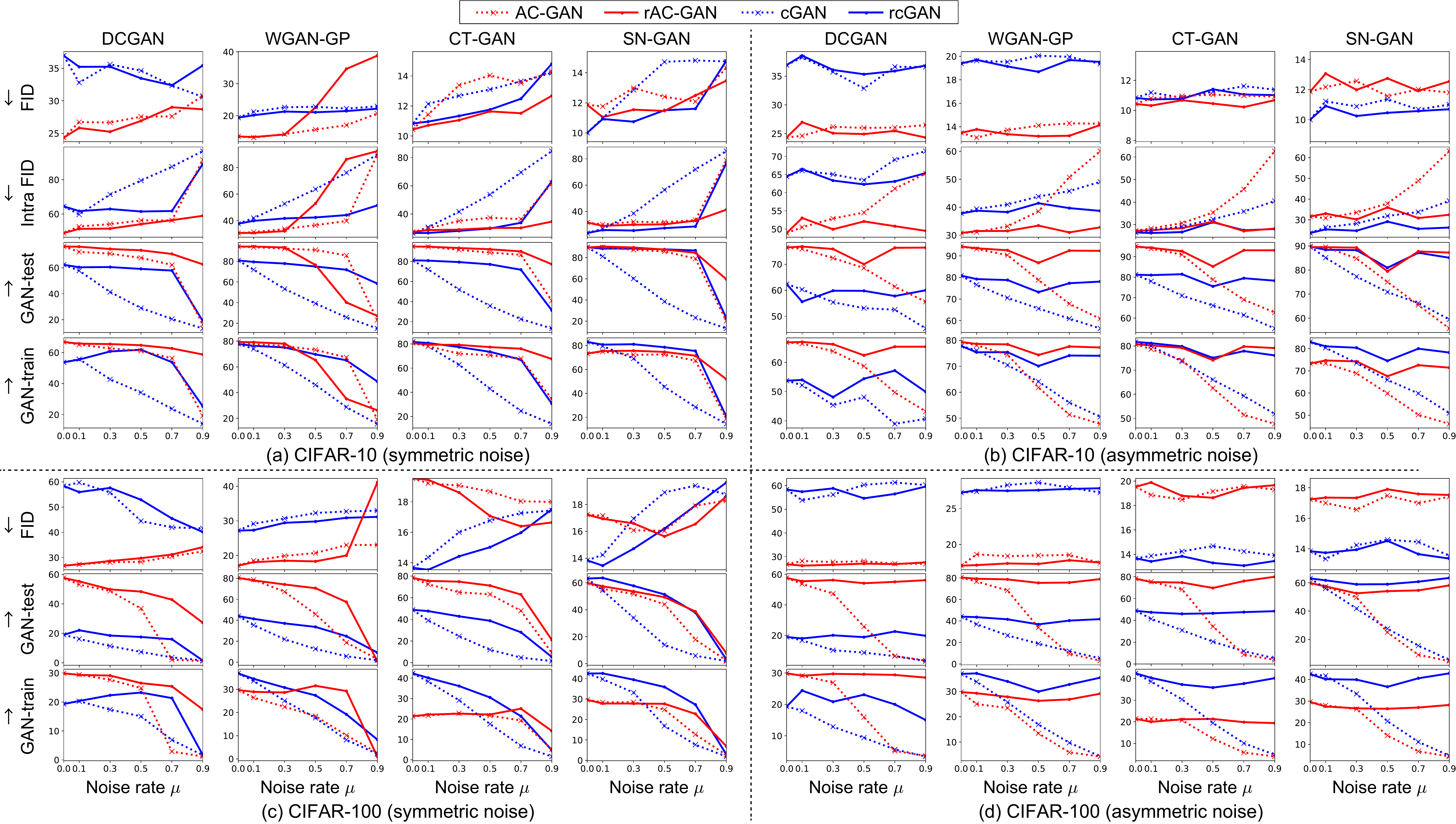}
  \caption{Quantitative results on CIFAR-10 and CIFAR-100. $\downarrow$ indicates the smaller the value, the better the performance. $\uparrow$ indicates the larger the value, the better the performance. Note that the scale is adjusted on each graph for easy viewing.}
  \vspace{-5mm}
  \label{fig:comp_eval}
\end{figure*}

\begin{table}[tb]
  \centering
  \scalebox{0.61}{
    \begin{tabular}{c|cccc}
      \bhline{1pt}
      & AC-CT-GAN & rAC-CT-GAN & cSN-GAN & rcSN-GAN
      \\ \bhline{0.75pt}
      cat/dog
      & 13.4/{\bf 83.9}
                  & {\bf 84.8}/10.3
                               & 35.6/{\bf 55.9}
                                         & {\bf 75.9}/13.0
      \\ \bhline{1pt}
    \end{tabular}
  }
  \vspace{0.5mm}
  \caption{Accuracy in the fourth column in Figure~\ref{fig:cifar10_gen}(b) (ground truth: cat) for the flipped classes (cat $\leftrightarrow$ dog)}
  \label{tab:acc}
  \vspace{-6mm}
\end{table}

\smallskip\noindent\textbf{Results.}
We present the quantitative results for each condition in Figure~\ref{fig:comp_eval} and provide a comparative summary between the proposed models (i.e., rAC-GAN and rcGAN) and the baselines (i.e., AC-GAN and cGAN) across all conditions in Figure~\ref{fig:vs}. We show the samples of generated images on CIFAR-10 with $\mu = 0.7$ in Figure~\ref{fig:cifar10_gen}. Regarding the FID (i.e., evaluating the quality of the overall generative distribution), the baselines and the proposed models are comparable in most cases, but when we use CT-GAN and SN-GAN (i.e., state-of-the-art models) in symmetric noise, the proposed models tend to outperform the baselines (32/40 conditions). This indicates that the label ambiguity caused by symmetric noise could disturb the learning of GANs if they have the high data-fitting ability. However, this degradation can be mitigated by using the proposed methods.

Regarding the other metrics (i.e., evaluating the quality of the conditional generative distribution), rAC-GAN and rcGAN tend to outperform AC-GAN and cGAN, respectively, across all the conditions. The one exception is rAC-WGAN-GP on CIFAR-10 with symmetric noise, but we find that it can be improved using the technique introduced in Section~\ref{subsec:imp}. We demonstrate this in Section~\ref{subsec:imp_eval}. Among the four models, CT-GAN and SN-GAN work relatively well for rAC-GAN and rcGAN, respectively. This tendency is also observed in clean label settings (i.e., $\mu = 0$). This indicates that the performance of rAC-GAN and rcGAN is closely related to the advance in the baseline GANs.

\begin{figure}[tb]
  \centering
  \includegraphics[width=\columnwidth]{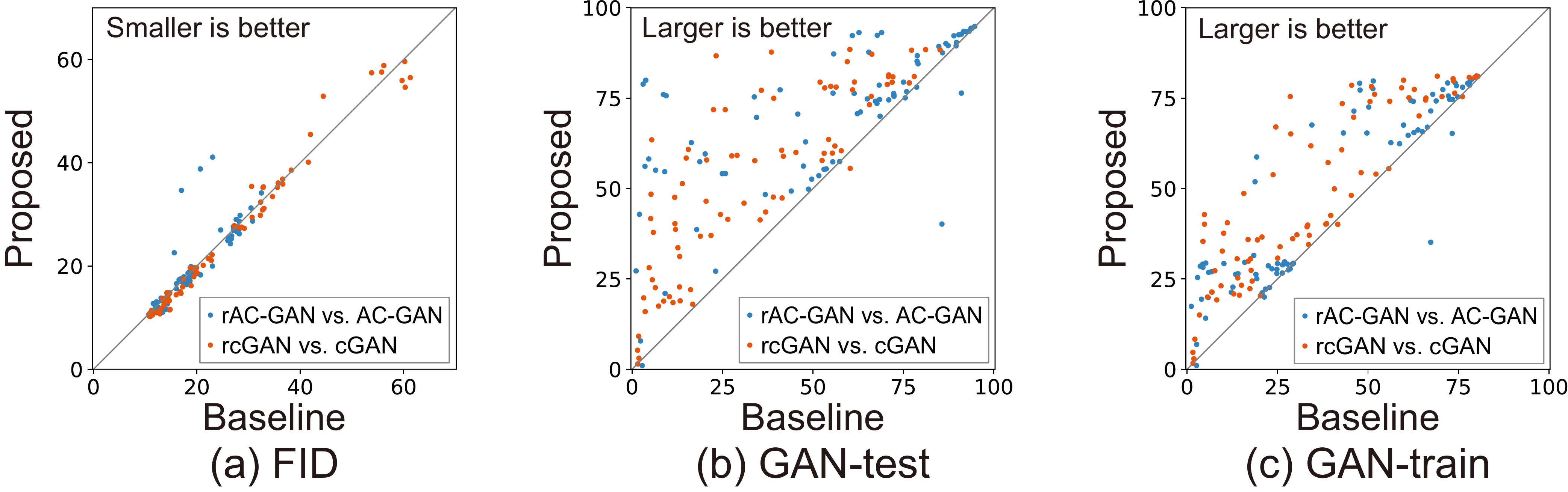}
  \caption{Comparison between the proposed models and the baselines
    across all the conditions in Figure \ref{fig:comp_eval}.}
  \label{fig:vs}
  \vspace{-3mm}
\end{figure}

\begin{table}[tb]
  \centering
  \scalebox{0.61}{
    \begin{tabular}{c|cccc}
      \bhline{1pt}
      & AC-GAN & rAC-GAN & cGAN & rcGAN
      \\ \bhline{0.75pt}
      Symmetric
      & -0.846 $\pm$ 0.084
               & -0.786 $\pm$ 0.163
                         & {\bf -0.989} $\pm$ 0.013
                                & -0.818 $\pm$ 0.142
      \\
      Asymmetric
      & {\bf -0.976} $\pm$ 0.008
               & -0.476 $\pm$ 0.119
                         & {\bf -0.985} $\pm$ 0.029
                                & -0.427 $\pm$ 0.274
      \\ \bhline{1pt}
    \end{tabular}
  }
  \vspace{0.5mm}
  \caption{Pearson correlation coefficient between the noise rate and GAN-train.
      The scores are averaged over all GAN configurations.}
  \label{tab:pearson}
  \vspace{-6mm}
\end{table}

An interesting finding is that the performance of cGAN in Intra FID, GAN-test, and GAN-train degrades linearly depending on the noise rate. To confirm this numerically, we calculated the Pearson correlation coefficient between the GAN-train and the noise rate. We list these in Table~\ref{tab:pearson}. These scores confirm that cGAN has the highest dependency on the noise rate, i.e., cGAN can fit even nosy labels. In contrast, AC-GAN shows robustness for symmetric noise but weakness for asymmetric noise. This would be related to the difficulty of memorization. In symmetric noise, the corruption variety is large, making it difficult to memorize labels. As a result, AC-GAN prioritizes learning simple (i.e., clean) labels, in a similar way as DNNs in image classification~\cite{DArpitICML2017}. In contrast, in asymmetric noise, the label corruption pattern is restrictive; as a result, AC-GAN easily fits noisy labels. Unlike AC-GAN, cGAN is a classifier-free model; therefore, cGAN tends to fit the given labels regardless of whether labels are noisy or not.

\begin{table*}
  \centering
  \scalebox{0.61}{
    \begin{tabular}{c|c|ccccc|ccccc|ccccc|ccccc}
      \bhline{1pt}
      \multirow{2}{*}{{Model}} & \multirow{2}{*}{{Metric}}
      & \multicolumn{5}{c|}{\!\!\!\! CIFAR-10 (symmetric noise) \!\!\!\!}
      & \multicolumn{5}{c|}{\!\!\!\! CIFAR-10 (asymmetric noise) \!\!\!\!}
      & \multicolumn{5}{c|}{\!\!\!\! CIFAR-100 (symmetric noise) \!\!\!\!}
      & \multicolumn{5}{c}{\!\!\!\! CIFAR-100 (asymmetric noise) \!\!\!\!}
      \\ \cline{3-22}
                               &
      & 0.1 & 0.3 & 0.5 & 0.7 & 0.9
      & 0.1 & 0.3 & 0.5 & 0.7 & 0.9
      & 0.1 & 0.3 & 0.5 & 0.7 & 0.9
      & 0.1 & 0.3 & 0.5 & 0.7 & 0.9
      \\ \bhline{0.75pt}
                               & FID $\downarrow$
      & 10.9 & 11.4 & 11.3 & 11.5 & 13.0                         
      & 10.8 & 10.2 & 10.2 & 10.4 & 11.0
      & 19.7 & 19.3 & 17.7 & 17.3 & 18.5
      & 19.4 & 19.3 & 19.7 & 18.8 & 19.0
      \\
      rAC-CT-GAN & Intra FID $\downarrow$
      & 28.7 & {\bf 31.0} & {\bf 30.1} & {\bf 31.7} & {\bf 38.9}
      & 28.5 & {\bf 27.4} & {\bf 31.2} & {\bf 35.0} & {\bf 36.8}
      & -- & -- & -- & -- & --
      & -- & -- & -- & -- & --
      \\
      with $T'$ & GAN-test $\uparrow$
      & 95.3 & 93.2 & {\bf 92.0} & 87.7 & {\bf 70.4}
      & 94.9 & 92.9 & {\bf 85.2} & {\bf 78.5} & {\bf 76.6}
      & {\bf 76.6} & 67.1 & {\bf 68.1} & \textit{1.0} & \textit{2.5}
      & 74.1 & 68.9 & \textit{28.7} & 7.2 & 2.2
      \\
                               & GAN-train $\uparrow$
      & 78.7 & {\bf 75.9} & {\bf 76.9} & {\bf 73.7} & {\bf 63.4}
      & 79.8 & {\bf 79.5} & {\bf 74.0} & {\bf 69.1} & {\bf 67.3}
      & 21.2 & 21.4 & 23.3 & \textit{1.0} & 2.3
      & 19.1 & 19.9 & 10.7 & 5.5 & 3.9
      \\ \hline
                               & FID $\downarrow$
      & 10.7 & 11.9 & 12.4 & 12.1 & 15.0
      & 10.8 & 10.8 & 11.0 & 10.9 & 11.3
      & 14.3 & 16.6 & 17.5 & 20.0 & 19.8
      & 13.8 & 14.1 & 14.7 & 14.7 & 13.9
      \\
      rcSN-GAN & Intra FID $\downarrow$
      & 25.5 & {\bf 29.4} & {\bf 29.4} & {\bf 29.7} & 87.4
      & 25.7 & 26.0 & {\bf 28.7} & 32.6 & {\bf 33.9}
      & -- & -- & -- & -- & --
      & -- & -- & -- & -- & --
      \\
      with $T'$ & GAN-test $\uparrow$
      & {\bf 85.3} & {\bf 79.0} & {\bf 84.8} & {\bf 82.8} & 15.9
      & 86.6 & {\bf 87.2} & {\bf 84.0} & {\bf 74.9} & {\bf 71.2}
      & 53.4 & 36.6 & {\bf 37.7} & \textit{1.0} & 1.7
      & {\bf 65.0} & {\bf 63.0} & {\bf 32.4} & \textit{7.8} & 3.8
      \\
                               & GAN-train $\uparrow$
      & 80.7 & {\bf 78.1} & {\bf 77.4} & {\bf 75.6} & 15.0
      & 80.5 & {\bf 79.0} & {\bf 75.7} & {\bf 69.3} & {\bf 65.7}
      & 40.1 & 32.8 & {\bf 31.3} & \textit{1.0} & 1.8
      & 41.7 & {\bf 39.3} & 20.1 & \textit{6.1} & 3.9
      \\ \bhline{1pt}
    \end{tabular}
  }
  \vspace{0.5mm}
  \caption{Quantitative results using the estimated $T'$. The second row indicates a noise rate. Bold and italic fonts indicate that the score is better or worse by more than 3 points over or under the baseline models (i.e., AC-CT-GAN or cSN-GAN), respectively. See Table~\ref{tab:estT_eval_ex} and Figure~\ref{fig:estT_gen} in the Appendix for more detailed comparison and generated image samples, respectively.}
  \label{tab:estT_eval}
  \vspace{-5mm}
\end{table*}

\subsection{Effects of estimated $T'$}
\label{subsec:estT_eval}

In Section~\ref{subsec:comp_eval}, we report the results using known $T$. As a more practical setting, we also evaluate our method with $T'$ estimated by a robust two-stage training algorithm~\cite{GPatriniCVPR2017}. We used CT-GAN for rAC-GAN and SN-GAN for rcGAN, which worked relatively well in both noisy and clean settings in Section~\ref{subsec:comp_eval}. We list the scores in Table~\ref{tab:estT_eval}. In CIFAR-10, even using $T'$, rAC-CT-GAN and rcSN-GAN tend to outperform conventional AC-CT-GAN and cSN-GAN, respectively, and show robustness to label noise. In CIFAR-100, when the noise rate is low, rAC-CT-GAN and rcSN-GAN work moderately well; however, in highly noisy settings, their performance is degraded. Note that such a tendency has also been observed in noisy label image classification with $T'$~\cite{GPatriniCVPR2017}, in which the authors argue that the high-rate mixture and limited number of images per class (500) make it difficult to estimate the correct $T$. Further improvement remains as an open issue.

\subsection{Evaluation of improved technique}
\label{subsec:imp_eval}

\begin{table}
  \centering
  \scalebox{0.61}{
    \begin{tabular}{c|c|cccc|cccc}
      \bhline{1pt}
      \multirow{2}{*}{Model}
      & \multirow{2}{*}{Metric}
      & \multicolumn{4}{c|}{CIFAR-10 (symmetric noise)}
      & \multicolumn{4}{c}{CIFAR-100 (symmetric noise)}
      \\ \cline{3-10}
      &
      & \,\,\, A \,\,\,
      & \,\,\, B \,\,\,
      & \,\,\, C \,\,\,
      & \,\,\, D \,\,\,
      & \,\,\, A \,\,\,
      & \,\,\, B \,\,\,
      & \,\,\, C \,\,\,
      & \,\,\, D \,\,\,
      \\ \bhline{0.75pt}
      & FID $\downarrow$
      & 27.9 & {\bf 14.7} & 12.4 & 13.5
      & 33.1 & {\bf 20.4} & 17.2 & 18.4
      \\
      Improved & Intra FID $\downarrow$
      & {\bf 55.7} & {\bf 34.6} & 33.4 & {\bf 36.9}
      & -- & -- & -- & --
      \\
      rAC-GAN & GAN-test $\uparrow$
      & 65.1 & {\bf 77.7} & 78.2 & {\bf 63.5}
      & 26.2 & {\bf 22.5} & 21.5 & {\bf 15.4}
      \\
      & GAN-train $\uparrow$
      & 59.9 & {\bf 70.8} & 69.1 & {\bf 59.7}
      & 17.1 & {\bf 16.3} & 14.8 & {\bf 11.7}
      \\ \hline
      & FID $\downarrow$
      & {\bf 30.4} & {\bf 16.9} & 14.2 & 14.9
      & \textit{50.2} & {\bf 25.8} & 18.0 & 18.7
      \\
      Improved & Intra FID $\downarrow$
      & {\bf 76.9} & {\bf 39.6} & {\bf 52.9} & {\bf 48.2}        
      & -- & -- & -- & --
      \\
      rcGAN & GAN-test $\uparrow$
      & {\bf 27.3} & {\bf 65.7} & {\bf 38.9} & {\bf 48.8}
      & {\bf 4.5} & 12.0 & {\bf 9.5} & {\bf 6.1}
      \\
      & GAN-train $\uparrow$
      & {\bf 31.9} & {\bf 60.7} & {\bf 36.7} & {\bf 47.3}
      & {\bf 6.0} & 10.3 & 7.5 & 4.4
      \\ \bhline{1pt}
    \end{tabular}
  }
  \vspace{0.5mm}
  \caption{Quantitative results using the improved technique. In the second row, A, B, C, and D indicate DCGAN, WGAN-GP, CT-GAN, and SN-GAN, respectively. We evaluated in severely noisy settings (i.e., symmetric noise with $\mu = 0.9$). Bold and italic fonts indicate that the score is better or worse by more than 3 points over or under naive models (i.e., rAC-GAN or rcGAN), respectively. See Table~\ref{tab:imp_eval_ex} and Figure~\ref{fig:imp_gen} in the Appendix for more detailed comparison and generated image samples, respectively.}
  \vspace{-1mm}
  \label{tab:imp_eval}
\end{table}

As shown in Figure~\ref{fig:comp_eval}, rAC-GAN and rcGAN show robustness for label noise in almost all cases, but we find that they are still weak to severely noisy settings (i.e., symmetric noise with $\mu = 0.9$) even though using known $T$. To improve the performance, we developed an improved technique (Section~\ref{subsec:imp}). In this section, we validate its effect. We list the scores in Table~\ref{tab:imp_eval}. We find that the improved degree depends on the GAN configurations, but, on the whole, the performance is improved by the proposed technique. In particular, we find that the improved technique is most effective for rAC-WGAN-GP, in which all the scores doubled compared to those of naive rAC-WGAN-GP.

\begin{table}
  \centering
  \scalebox{0.61}{
    \begin{tabular}{c|cc|cc|cc|cc|cc}
      \bhline{1pt}
      \multirow{2}{*}{Metric}
      & \multicolumn{2}{c|}{Clean}
      & \multicolumn{4}{c|}{Noisy}
      & \multicolumn{4}{c}{Mixed}
      \\ \cline{2-11}
      & AC & c
      & AC & rAC & c & rc
           & AC & rAC & c & rc
      \\ \bhline{0.75pt}
      FID $\downarrow$
      & 6.8 & 12.0
      & {\bf 4.4} & 4.6 & 9.4 & 9.4
           & 4.8 & {\bf 4.7} & 10.5 & {\bf 9.7}
      \\
      GAN-train $\uparrow$
      & 56.6 & 53.9
      & 49.5 & {\bf 51.7} & 48.6 & {\bf 49.8}
           & 52.8 & {\bf 57.0} & 51.7 & {\bf 55.0}
      \\ \bhline{1pt}
    \end{tabular}
  }
  \vspace{0.5mm}
  \caption{Quantitative results on Clothing1M. AC, rAC, c, and rc denote AC-CT-GAN, rAC-CT-GAN, cSN-GAN, and rcSN-GAN, respectively. Bold font indicates better scores in each block. See Figure~\ref{fig:clothing1m_gen} in the Appendix for generated image samples.}
  \label{tab:clothing1m_eval}
  \vspace{-5mm}
\end{table}

\subsection{Evaluation on real-world noise}
\label{subsec:clothing1m_eval}

Finally, we tested on Clothing1M~\cite{TXiaoCVPR2015} to analyze the effectiveness on real-world noise.\footnote{We tested 10 conditions in total. For each condition, we trained three models with different initializations and report the results averaged over them.} Clothing1M contains $1M$ clothing images in 14 classes. The data are collected from several online shopping websites and include many mislabeled samples. This dataset also contains $50k$, $14k$, and $10k$ of clean data for training, validation, and testing, respectively. Following the previous studies~\cite{TXiaoCVPR2015,GPatriniCVPR2017}, we approximated $T$ using the partial ($25k$) training data that have both clean and noisy labels. We tested on three settings: (1) $50k$ {\bf clean} data, (2) $1M$ {\bf noisy} data, and (3) {\bf mixed} data that consists of clean data (bootstrapped to $500k$) and $1M$ noisy data, which are used in~\cite{TXiaoCVPR2015} to boost the performance of image classification. We used AC-CT-GAN/rAC-CT-GAN and cSN-GAN/rcSN-GAN. We resized images from $256 \times 256$ to $64 \times 64$ to shorten the training time. 

\smallskip\noindent\textbf{Results.}
We list the scores in Table~\ref{tab:clothing1m_eval}.\footnote{We did not use Intra FID because the number of clean labeled data for each class is few. We did not use the GAN-test because this dataset is challenging and a trained classifier tends to be deceived by noisy data.} The comparison of FID values indicates that the scores depend on the number of data ({\bf noisy}, {\bf mixed} $>$ {\bf clean}) rather than the difference between the baseline and proposed models. This suggests that, in this type noise setting, the scale of the dataset should be made large, even though labels are noisy, to capture an overall distribution. In contrast, the comparison of the GAN-train between the clean and noisy data settings indicates the importance of label accuracy. In the noisy data setting, the scores improve using rAC-GAN or rcGAN but they are still worse than those using AC-GAN and cGAN in the clean data setting. The balanced models are rAC-GAN and rcGAN in the mixed data setting. They are comparable to the models in the noisy data setting in terms of the FID and outperform the models in the clean data setting in terms of the GAN-train. Recently, data augmentation~\cite{MFAdarISBI2018,ZZhangCVPR2018} has been studied intensively as an application of conditional generative models. We expect the above findings to provide an important direction in this space.

\section{Conclusion}
\label{sec:conclusion}

Recently, conditional extensions of GANs have shown promise in image generation; however, the limitation here is that they need large-scale accurate class-labeled data to be available. To remedy this, we developed a new family of GANs called rGANs that incorporate a noise transition model into conditional extensions of GANs. In particular, we introduced two variants: rAC-GAN, which is a bridging model between GANs and the noise-robust classification models, and rcGAN, which is an extension of cGAN and solves this problem with no reliance on any classifier. In addition to providing a theoretical background, we demonstrate the effectiveness and limitations of the proposed models through extensive experiments in various settings. In the future, we hope that our findings facilitate the construction of a conditional generative model in real-world scenarios in which only noisy labeled data are available.

\renewcommand{\baselinestretch}{1}\selectfont
\section*{Acknowledgement}
We thank Hiroharu Kato, Yusuke Mukuta, and Mikihiro Tanaka for helpful discussions. This work was supported by JSPS KAKENHI Grant Number JP17H06100, partially supported by JST CREST Grant Number JPMJCR1403, Japan, and partially supported by the Ministry of Education, Culture, Sports, Science and Technology (MEXT) as ``Seminal Issue on Post-K Computer.''

{\small
  \bibliographystyle{ieee}
  \bibliography{refs}
}

\clearpage
\appendix
\section{Contents}
\label{sec:contents}

\begin{itemize}
  \setlength{\parskip}{3pt}
  \setlength{\itemsep}{3pt}
\item Appendix~\ref{sec:ex} (pp. \pageref{sec:ex}--\pageref{fig:clothing1m_gen})
  \begin{itemize}
  \item We provide the extended results of Sections~\ref{subsec:comp_eval}--\ref{subsec:clothing1m_eval} in Appendices~\ref{subsec:comp_eval_ex}--\ref{subsec:clothing1m_eval_ex}, respectively.
  \end{itemize}
\item Appendix~\ref{sec:ana} (pp. \pageref{sec:ana}--\pageref{fig:intra_fid_noise})
  \begin{itemize}
  \item We provide an additional analysis.
  \end{itemize}
\item Appendices~\ref{sec:comp_eval_detail}--\ref{sec:clothing1m_eval_detail} (pp. \pageref{sec:comp_eval_detail}--\pageref{subsec:clothing1m_eval_detail})
  \begin{itemize}
  \item We describe the details of the experimental setup of Sections~\ref{subsec:comp_eval}--\ref{subsec:clothing1m_eval} in Appendices~\ref{sec:comp_eval_detail}--\ref{sec:clothing1m_eval_detail}, respectively.
  \end{itemize}
\end{itemize}

\section{Extended results}
\label{sec:ex}

\subsection{Extended results of Section~\ref{subsec:comp_eval}}
\label{subsec:comp_eval_ex}

As extended results of Section~\ref{subsec:comp_eval} (comprehensive study), we provide the image samples generated using the models evaluated in Section~\ref{subsec:comp_eval} in Figures~\ref{fig:dcgan_gen}--\ref{fig:sngan_gen}. An outline of the content is as follows:
\begin{itemize}
  \setlength{\parskip}{3pt}
  \setlength{\itemsep}{3pt}
\item Figure~\ref{fig:dcgan_gen}: Image samples generated using DCGANs (AC-DCGAN, rAC-DCGAN, cDCGAN, and rcDCGAN) on CIFAR-10
\item Figure~\ref{fig:wgangp_gen}: Image samples generated using WGAN-GPs (AC-WGAN-GP, rAC-WGAN-GP, cWGAN-GP, and rcWGAN-GP) on CIFAR-10
\item Figure~\ref{fig:ctgan_gen}: Image samples generated using CT-GANs (AC-CT-GAN, rAC-CT-GAN, cCT-GAN, and rcCT-GAN) on CIFAR-10
\item Figure~\ref{fig:sngan_gen}: Image samples generated using SN-GANs (AC-SN-GAN, rAC-SN-GAN, cSN-GAN, and rcSN-GAN) on CIFAR-10
\end{itemize}

\subsection{Extended results of Section~\ref{subsec:estT_eval}}
\label{subsec:estT_eval_ex}

We provide the extended results of Section~\ref{subsec:estT_eval} (effects of estimated $T'$) in Table~\ref{tab:estT_eval_ex}, Figure~\ref{fig:estT_eval_ex}, and Figure~\ref{fig:estT_gen}. An outline of the content is as follows:
\begin{itemize}
  \setlength{\parskip}{3pt}
  \setlength{\itemsep}{3pt}
\item Table~\ref{tab:estT_eval_ex}: Quantitative results using estimated $T'$ (extended version of Table~\ref{tab:estT_eval})
\item Figure~\ref{fig:estT_eval_ex}: Visualization of Table~\ref{tab:estT_eval_ex}
\item Figure~\ref{fig:estT_gen}: Image samples generated using rAC-CT-GAN with estimated $T'$ and rcSN-GAN with estimated $T'$ on CIFAR-10
\end{itemize}

\subsection{Extended results of Section~\ref{subsec:imp_eval}}
\label{subsec:imp_eval_ex}

We provide the extended results of Section~\ref{subsec:imp_eval} (evaluation of the improved technique) in Table~\ref{tab:imp_eval_ex} and Figure~\ref{fig:imp_gen}. An outline of the content is as follows:
\begin{itemize}
  \setlength{\parskip}{3pt}
  \setlength{\itemsep}{3pt}
\item Table~\ref{tab:imp_eval_ex}: Quantitative results using the improved technique (extended version of Table~\ref{tab:imp_eval})
\item Figure~\ref{fig:imp_gen}: Image samples generated using the improved rAC-GANs and improved rcGANs with combinations of DCGAN, WGAN-GP, CT-GAN, and SN-GAN  on CIFAR-10
\end{itemize}

\subsection{Extended results of Section~\ref{subsec:clothing1m_eval}}
\label{subsec:clothing1m_eval_ex}

We provide the extended results of Section~\ref{subsec:clothing1m_eval} (evaluation on the real-world noise) in Figure~\ref{fig:clothing1m_gen}. An outline of the content is as follows:
\begin{itemize}
  \setlength{\parskip}{3pt}
  \setlength{\itemsep}{3pt}
\item Figure~\ref{fig:clothing1m_gen}: Image samples generated using AC-CT-GAN, rAC-CT-GAN, cSN-GAN, and rcSN-GAN on Clothing1M (clean, noisy, and mixed settings)
\end{itemize}

\clearpage
\begin{figure*}[t]
  \centering
  \includegraphics[width=\textwidth]{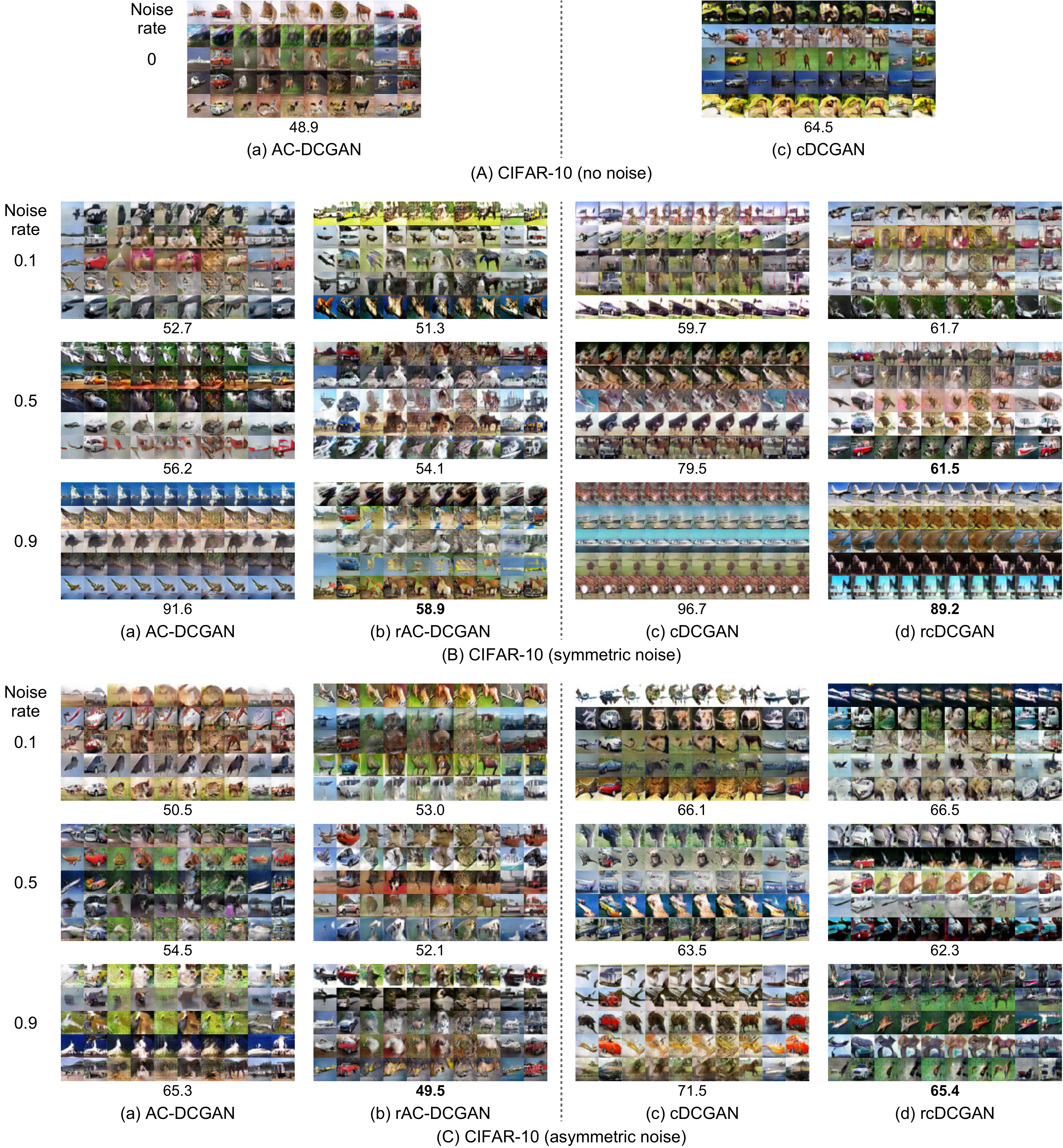}
  \caption{Image samples generated using (a) AC-DCGAN, (b) rAC-DCGAN, (c) cDCGAN, and (d) rcDCGAN on CIFAR-10 ((A) no noise, (B) symmetric noise, and (C) asymmetric noise). These models are discussed in Section~\ref{subsec:comp_eval}. In each picture block, each column shows samples associated with the same class. Each row includes samples generated from a fixed ${\bm z}$ and a varied $y^g$. The value below each picture block represents the achieved Intra FID (which is the same as the value reported in Figure~\ref{fig:comp_eval}). The smaller the value, the better. When the score difference between the baseline models (AC-DCGAN and cDCGAN) and the proposed models (rAC-DCGAN and rcDCGAN) is more than 3 points, we use bold font to indicate the better model.}
  \label{fig:dcgan_gen}
\end{figure*}

\begin{figure*}[t]
  \centering
  \includegraphics[width=\textwidth]{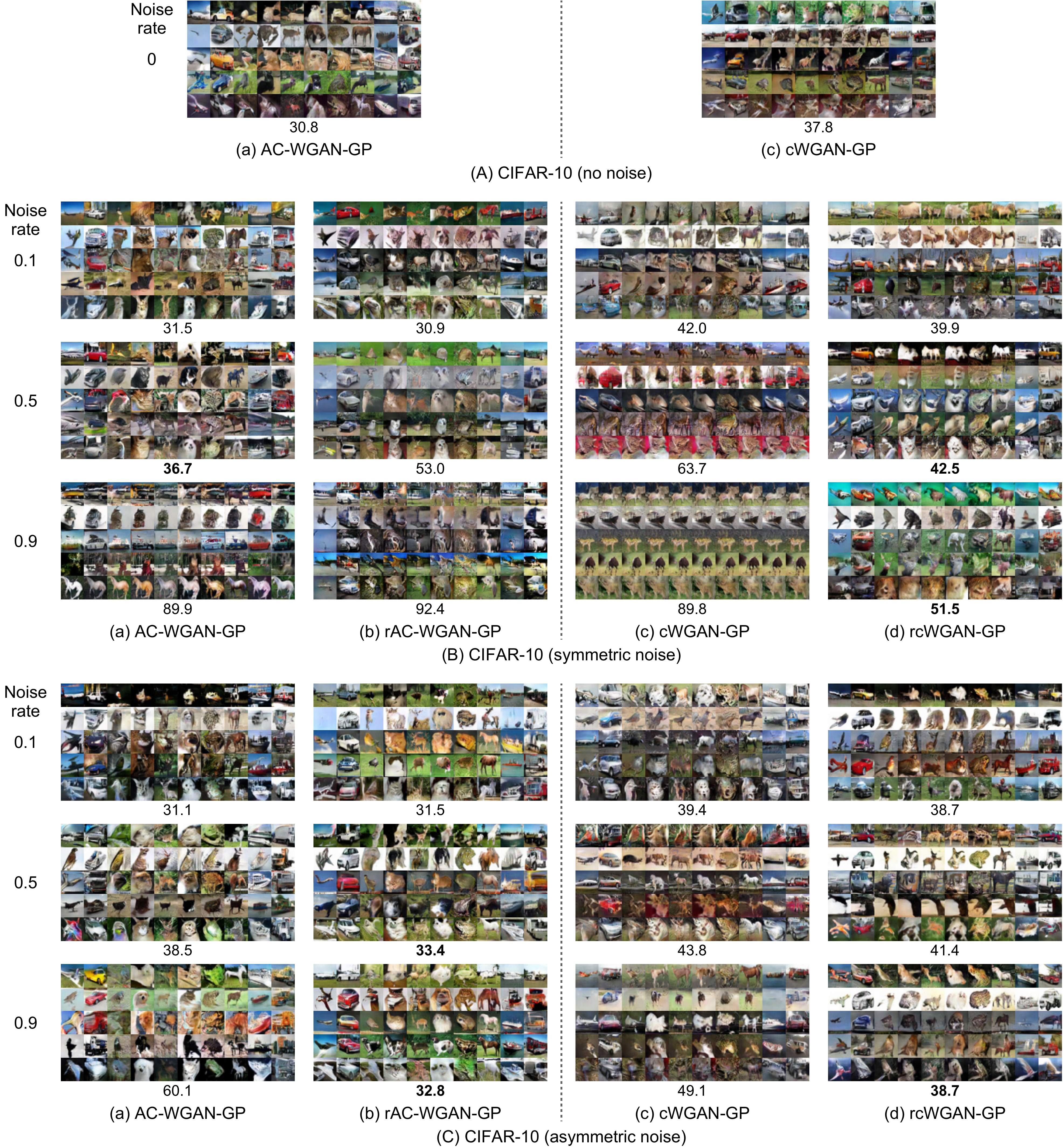}
  \caption{Image samples generated using (a) AC-WGAN-GP, (b) rAC-WGAN-GP, (c) cWGAN-GP, and (d) rcWGAN-GP on CIFAR-10 ((A) no noise, (B) symmetric noise, and (C) asymmetric noise). These models are discussed in Section~\ref{subsec:comp_eval}. In each picture block, each column shows samples associated with the same class. Each row includes samples generated from a fixed ${\bm z}$ and a varied $y^g$. The value below each picture block represents the achieved Intra FID (which is the same as the value reported in Figure~\ref{fig:comp_eval}). The smaller the value, the better. When the score difference between the baseline models (AC-WGAN-GP and cWGAN-GP) and the proposed models (rAC-WGAN-GP and rcWGAN-GP) is more than 3 points, we use bold font to indicate the better model.}
  \label{fig:wgangp_gen}
\end{figure*}

\begin{figure*}[t]
  \centering
  \includegraphics[width=\textwidth]{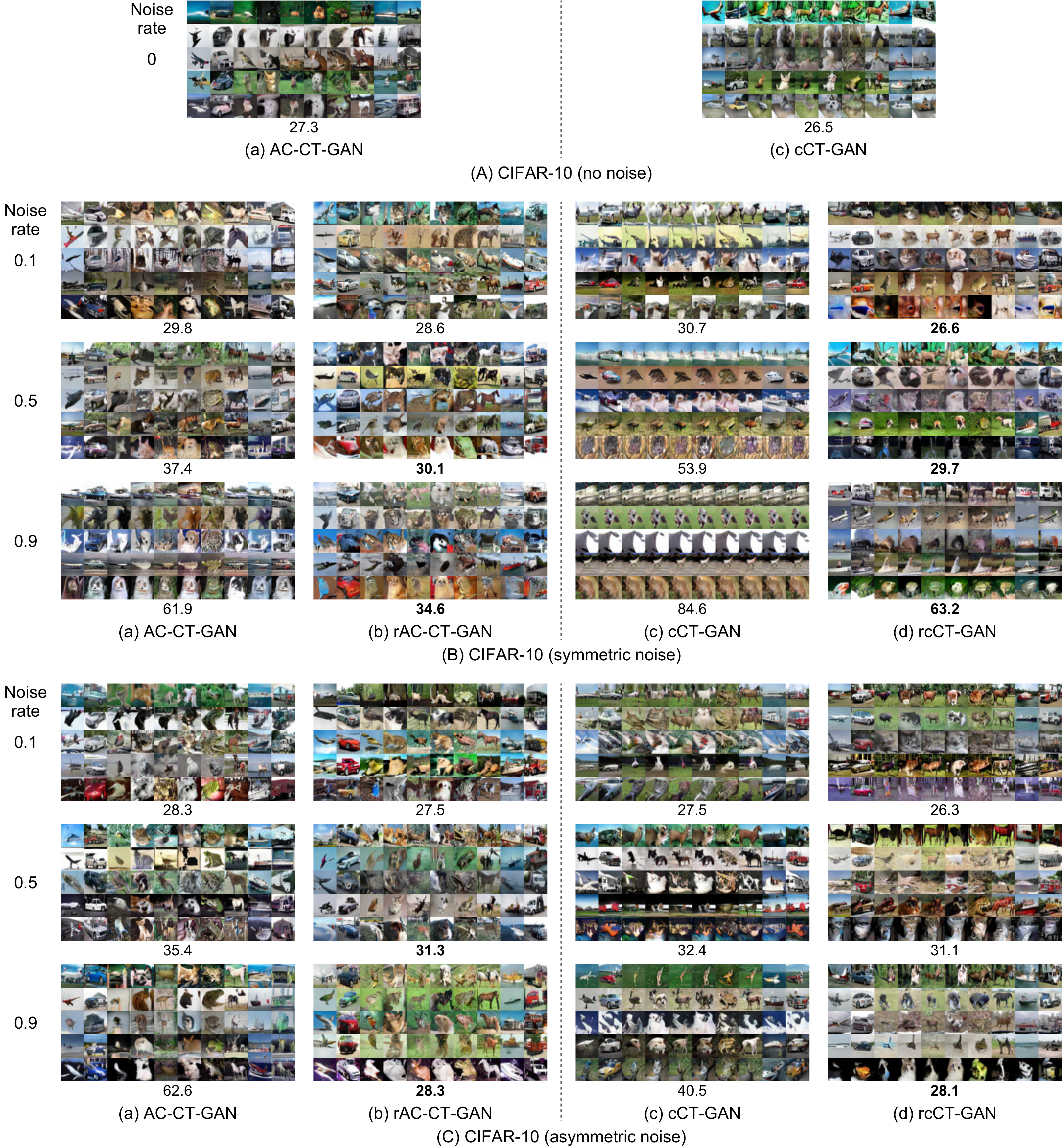}
  \caption{Image samples generated using (a) AC-CT-GAN, (b) rAC-CT-GAN, (c) cCT-GAN, and (d) rcCT-GAN on CIFAR-10 ((A) no noise, (B) symmetric noise, and (C) asymmetric noise). These models are discussed in Section~\ref{subsec:comp_eval}. In each picture block, each column shows samples associated with the same class. Each row includes samples generated from a fixed ${\bm z}$ and a varied $y^g$. The value below each picture block represents the achieved Intra FID (which is the same as the value reported in Figure~\ref{fig:comp_eval}). The smaller the value, the better. When the score difference between the baseline models (AC-CT-GAN and cCT-GAN) and the proposed models (rAC-CT-GAN and rcCT-GAN) is more than 3 points, we use bold font to indicate the better model.}
  \label{fig:ctgan_gen}
\end{figure*}

\begin{figure*}[t]
  \centering
  \includegraphics[width=\textwidth]{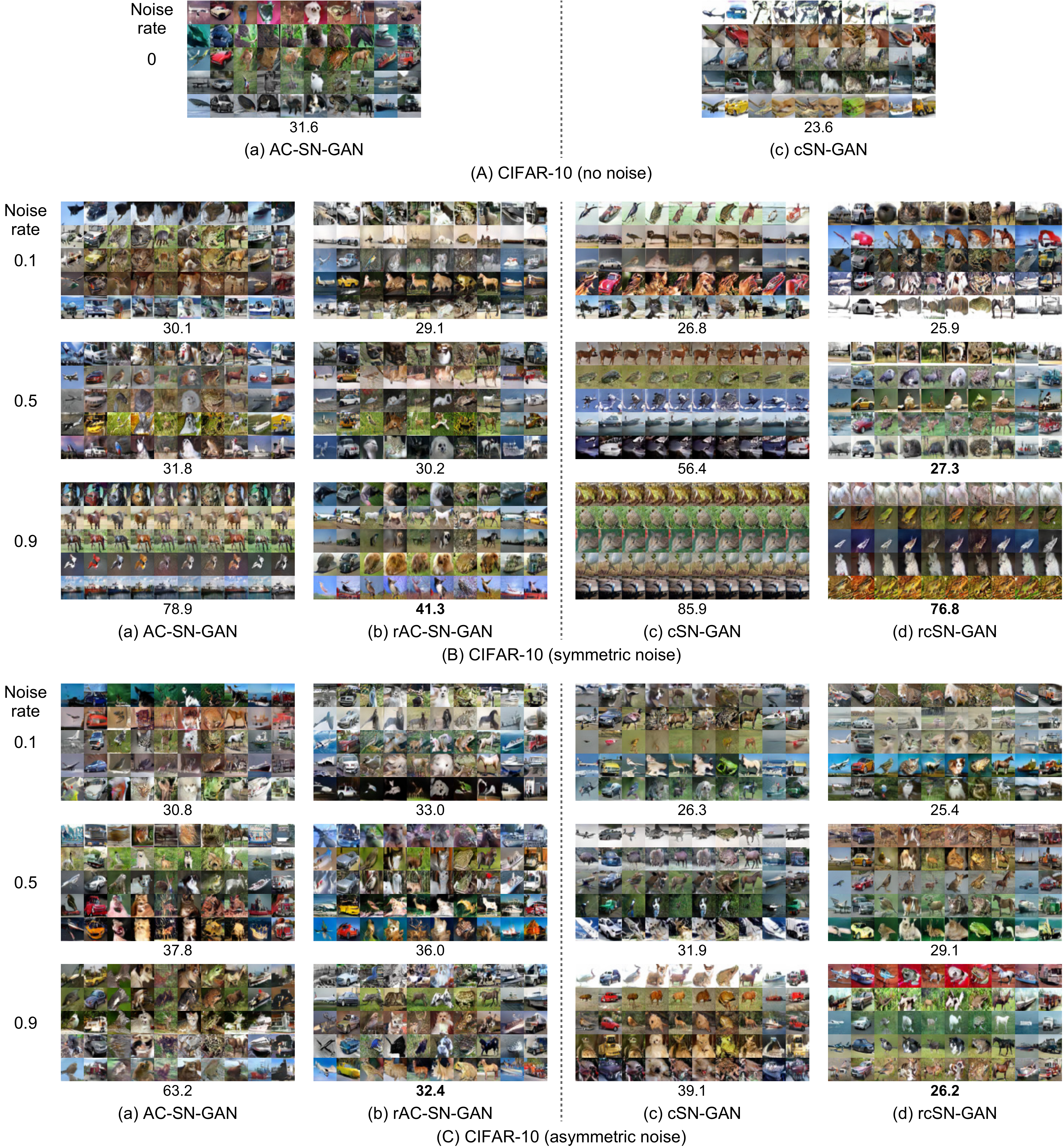}
  \caption{Image samples generated using (a) AC-SN-GAN, (b) rAC-SN-GAN, (c) cSN-GAN, and (d) rcSN-GAN on CIFAR-10 ((A) no noise, (B) symmetric noise, and (C) asymmetric noise). These models are discussed in Section~\ref{subsec:comp_eval}. In each picture block, each column shows samples associated with the same class. Each row includes samples generated from a fixed ${\bm z}$ and a varied $y^g$. The value below each picture block represents the achieved Intra FID (which is the same as the value reported in Figure~\ref{fig:comp_eval}). The smaller the value, the better. When the score difference between the baseline models (AC-SN-GAN and cSN-GAN) and the proposed models (rAC-SN-GAN and rcSN-GAN) is more than 3 points, we use bold font to indicate the better model.}
  \label{fig:sngan_gen}
\end{figure*}

\clearpage
\begin{table*}[hbt]
  \centering
  \scalebox{0.8}{
    \begin{tabular}{c|c|ccccc|ccccc}
      \bhline{1pt}
      \multirow{2}{*}{{Model}} & \multirow{2}{*}{{Metric}}
      & \multicolumn{5}{c|}{CIFAR-10 (symmetric noise)}
      & \multicolumn{5}{c}{CIFAR-10 (asymmetric noise)}
      \\ \cline{3-12}
                               &
      & 0.1 & 0.3 & 0.5 & 0.7 & 0.9
                  & 0.1 & 0.3 & 0.5 & 0.7 & 0.9      
      \\ \bhline{0.75pt}
                               & \multirow{2}{*}{FID $\downarrow$}
      & 10.9 & 11.4 & 11.3 & 11.5 & 13.0
                  & 10.8 & 10.2 & 10.2 & 10.4 & 11.0
      \\      
                               &
      & (11.4) & (13.4) & (14.0) & (13.5) & (14.3)
                  & (10.8) & (11.0) & (11.0) & (11.0) & (10.9)
      \\ \cline{2-12}      
                               & \multirow{2}{*}{Intra FID $\downarrow$}
      & 28.7 & {\bf 31.0} & {\bf 30.1} & {\bf 31.7} & {\bf 38.9}
                  & 28.5 & {\bf 27.4} & {\bf 31.2} & {\bf 35.0} & {\bf 36.8}
      \\
      rAC-CT-GAN with $T'$
                               &
      & (29.8) & (35.1) & (37.4) & (36.4) & (61.9)
                  & (28.3) & (30.7) & (35.4) & (45.7) & (62.6)
      \\ \cline{2-12}
      (AC-CT-GAN)
                               & \multirow{2}{*}{GAN-test $\uparrow$}
      & 95.3 & 93.2 & {\bf 92.0} & 87.7 & {\bf 70.4}
                  & 94.9 & 92.9 & {\bf 85.2} & {\bf 78.5} & {\bf 76.6}
      \\      
                               &
      & (94.7) & (91.7) & (88.9) & (86.7) & (40.9)
                  & (94.0) & (91.0) & (78.8) & (69.0) & (62.7)
      \\ \cline{2-12}      
                               & \multirow{2}{*}{GAN-train $\uparrow$}
      & 78.7 & {\bf 75.9} & {\bf 76.9} & {\bf 73.7} & {\bf 63.4}
                  & 79.8 & {\bf 79.5} & {\bf 74.0} & {\bf 69.1} & {\bf 67.3}
      \\
                               &
      & (78.1) & (72.0) & (70.7) & (67.9) & (34.5)
                  & (78.7) & (74.1) & (62.5) & (51.5) & (47.7)
      \\ \hline
                               & \multirow{2}{*}{FID $\downarrow$}
      & 10.7 & 11.9 & 12.4 & 12.1 & 15.0
                  & 10.8 & 10.8 & 11.0 & 10.9 & 11.3
      \\
                               &
      & (11.0) & (12.9) & (14.7) & (14.8) & (14.8)
                  & (11.2) & (10.9) & (11.4) & (10.7) & (11.0)
      \\ \cline{2-12}      
                               & \multirow{2}{*}{Intra FID $\downarrow$}
      & 25.5 & {\bf 29.4} & {\bf 29.4} & {\bf 29.7} & 87.4
                  & 25.7 & 26.0 & {\bf 28.7} & 32.6 & {\bf 33.9}
      \\
      rcSN-GAN with $T'$
                               &
      & (26.8) & (38.5) & (56.4) & (72.0) & (85.9)
                  & (26.3) & (28.2) & (31.9) & (33.7) & (39.1)
      \\ \cline{2-12}
      (cSN-GAN)
                               & \multirow{2}{*}{GAN-test $\uparrow$}
      & {\bf 85.3} & {\bf 79.0} & {\bf 84.8} & {\bf 82.8} & 15.9
                  & 86.6 & {\bf 87.2} & {\bf 84.0} & {\bf 74.9} & {\bf 71.2}
      \\
                               &
      & (81.1) & (60.2) & (38.5) & (23.2) & (13.1)
                  & (85.1) & (77.3) & (70.8) & (66.3) & (59.5)
      \\ \cline{2-12}
                               & \multirow{2}{*}{GAN-train $\uparrow$}
      & 80.7 & {\bf 78.1} & {\bf 77.4} & {\bf 75.6} & 15.0
                  & 80.5 & {\bf 79.0} & {\bf 75.7} & {\bf 69.3} & {\bf 65.7}
      \\
                               &
      & (79.5) & (69.2) & (45.5) & (28.5) & (14.5)
                  & (80.4) & (73.5) & (65.9) & (59.8) & (51.0)
      \\ \bhline{1pt}
      \multicolumn{12}{c}{} \vspace{-2mm}
      \\
      \multicolumn{12}{c}{(a) CIFAR-10}
      \\ \multicolumn{12}{c}{}
      \\ \bhline{1pt}
      \multirow{2}{*}{{Model}} & \multirow{2}{*}{{Metric}}
      & \multicolumn{5}{c|}{CIFAR-100 (symmetric noise)}
      & \multicolumn{5}{c}{CIFAR-100 (asymmetric noise)}
      \\ \cline{3-12}
                               &
      & 0.1 & 0.3 & 0.5 & 0.7 & 0.9
                  & 0.1 & 0.3 & 0.5 & 0.7 & 0.9
      \\ \bhline{0.75pt}
                               & \multirow{2}{*}{FID $\downarrow$}
      & 19.7 & 19.3 & 17.7 & 17.3 & 18.5
                  & 19.4 & 19.3 & 19.7 & 18.8 & 19.0
      \\
                               &
      & (19.2) & (19.1) & (18.7) & (18.0) & (18.0)
                  & (18.9) & (18.5) & (19.2) & (19.6) & (19.3)
      \\ \cline{2-12}      
      rAC-CT-GAN with $T'$
                               & \multirow{2}{*}{GAN-test $\uparrow$}
      & {\bf 76.6} & 67.1 & {\bf 68.1} & \textit{1.0} & \textit{2.5}
                  & 74.1 & 68.9 & \textit{28.7} & 7.2 & 2.2
      \\
      (AC-CT-GAN)
                               &
      & (72.4) & (65.0) & (63.1) & (48.0) & (9.1)
                  & (75.5) & (68.4) & (34.4) & (8.7) & (3.8)
      \\ \cline{2-12}      
                               & \multirow{2}{*}{GAN-train $\uparrow$}
      & 21.2 & 21.4 & 23.3 & \textit{1.0} & 2.3
                  & 19.1 & 19.9 & 10.7 & 5.5 & 3.9
      \\
                               &
      & (21.7) & (22.8) & (21.7) & (19.3) & (5.1)
                  & (21.4) & (20.8) & (12.2) & (5.8) & (4.0)
      \\ \hline
                               & \multirow{2}{*}{FID $\downarrow$}
      & 14.3 & 16.6 & 17.5 & 20.0 & 19.8
                  & 13.8 & 14.1 & 14.7 & 14.7 & 13.9
      \\      
                               &
      & (14.2) & (16.9) & (18.9) & (19.4) & (18.7)
                  & (13.3) & (14.2) & (14.6) & (14.4) & (13.5)
      \\ \cline{2-12}
      rcSN-GAN with $T'$
                               & \multirow{2}{*}{GAN-test $\uparrow$}
      & 53.4 & 36.6 & {\bf 37.7} & \textit{1.0} & 1.7
                  & {\bf 65.0} & {\bf 63.0} & {\bf 32.4} & \textit{7.8} & 3.8
      \\
      (rcSN-GAN)
                               &
      & (54.3) & (33.9) & (13.9) & (5.9) & (1.9)
                  & (56.1) & (41.8) & (27.5) & (15.6) & (5.4)
      \\ \cline{2-12}
                               & \multirow{2}{*}{GAN-train $\uparrow$}
      & 40.1 & 32.8 & {\bf 31.3} & \textit{1.0} & 1.8
                  & 41.7 & {\bf 39.3} & 20.1 & \textit{6.1} & 3.9
      \\
                               &
      & (39.7) & (33.2) & (16.9) & (7.7) & (1.9)
                  & (41.7) & (33.3) & (20.7) & (11.1) & (4.8)
      \\ \bhline{1pt}
      \multicolumn{12}{c}{} \vspace{-2mm}
      \\
      \multicolumn{12}{c}{(b) CIFAR-100}
      \\
    \end{tabular}
  }
  \vspace{2mm}
  \caption{Extended version of Table~\ref{tab:estT_eval}. Quantitative results using the estimated $T'$. These results are discussed in Section~\ref{subsec:estT_eval}. In each table, the second row indicates a noise rate $\mu \in \{ 0.1, 0.3, 0.5, 0.7, 0.9 \}$. Under the third row, each odd row contains the scores for the proposed models (i.e., rAC-CT-GAN or rcSN-GAN) with $T'$ and each even row (denoted in parenthesis) includes the scores for the baseline models (i.e., AC-CT-GAN or cSN-GAN). Bold and italic fonts indicate that the score for the proposed models is better or worse by more than 3 points than that for the baseline models, respectively. See also Figure~\ref{fig:estT_eval_ex} that visualizes this information as graphs.}
  \label{tab:estT_eval_ex}
\end{table*}

\begin{figure*}[t]
  \centering
  \includegraphics[width=0.85\textwidth]{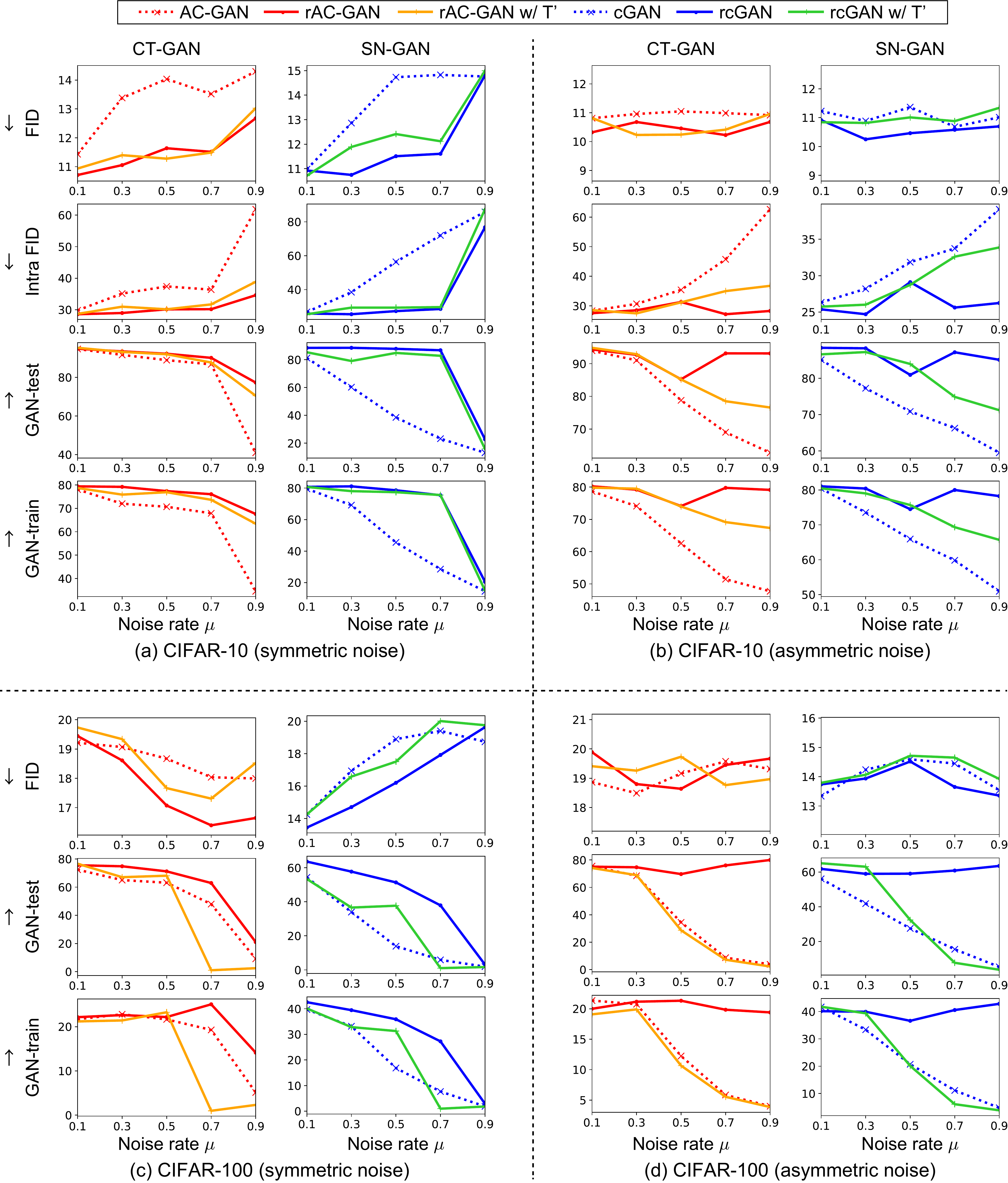}
  \caption{Visualization of Table~\ref{tab:estT_eval_ex}. Comparison of the quantitative results using the baseline models (AC-CT-GAN and cSN-GAN), the proposed models (rAC-CT-GAN and rcSN-GAN) with the known $T$, and the proposed models with the estimated $T'$. The scale is adjusted on each graph for easy viewing. As discussed in Section~\ref{subsec:estT_eval}, in CIFAR-10, even using $T'$, rAC-CT-GAN and rcSN-GAN outperform conventional AC-CT-GAN and cSN-GAN, respectively, and show robustness to label noise. Furthermore, rAC-CT-GAN and rcSN-GAN with $T$ and those with $T'$ are almost similar except for the asymmetric noise with a higher noise rate (i.e., 0.7 and 0.9). In CIFAR-100, when the noise rate is low, rAC-CT-GAN and rcSN-GAN work moderately well; however, in highly noisy settings, their performance is degraded. This implies the limitation of estimating $T'$ from the data in which there is a high-rate mixture and there is a limited number of images per class (500). This is also mentioned in the previous study~\cite{GPatriniCVPR2017}. Further improvement remains as an open issue. The precise values for this figure are provided in Table~\ref{tab:estT_eval_ex}.}
  \label{fig:estT_eval_ex}
\end{figure*}

\begin{figure*}[t]
  \centering
  \includegraphics[width=\textwidth]{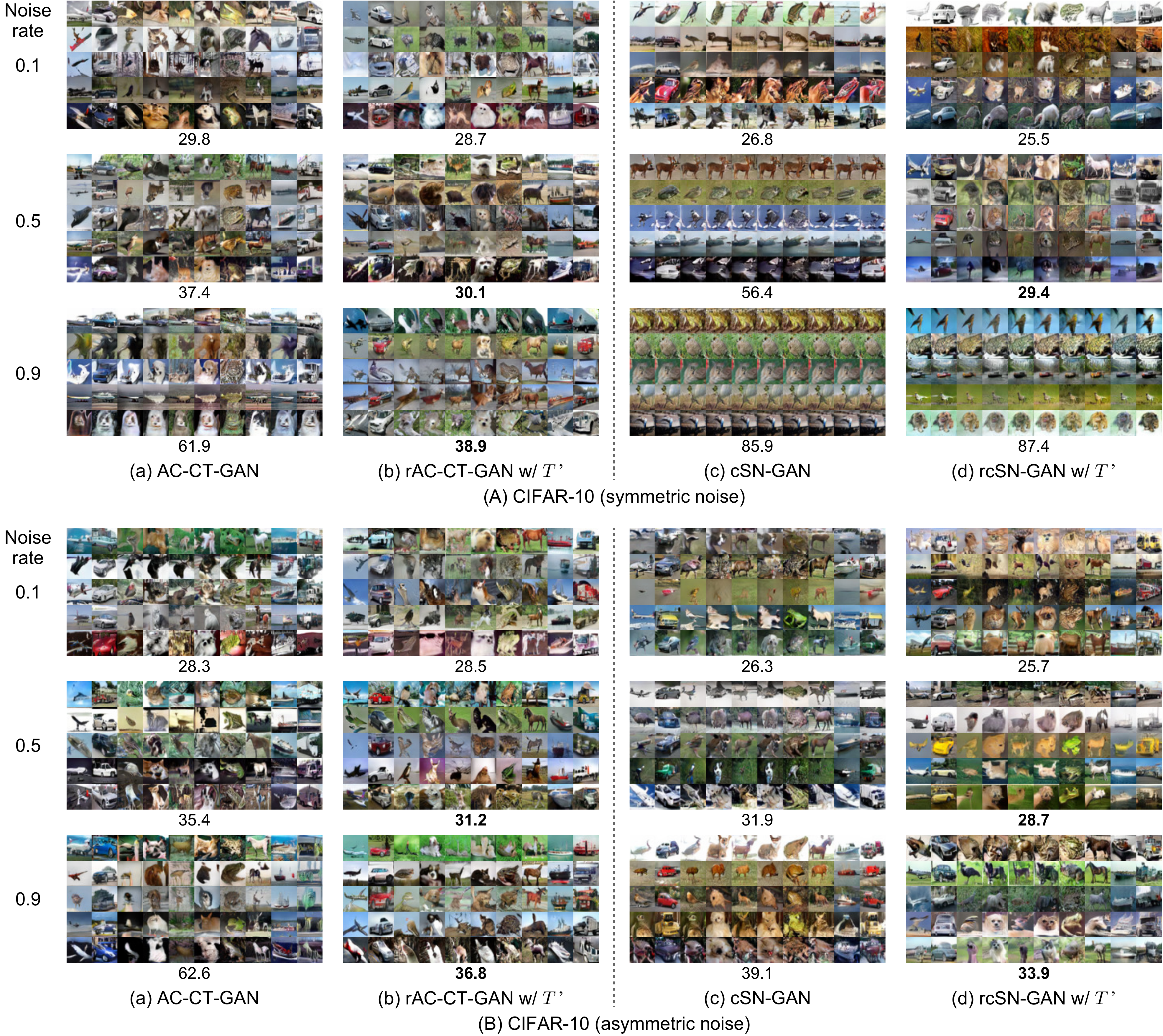}
  \caption{Image samples generated using (a) AC-CT-GAN, (b) rAC-CT-GAN with the estimated $T'$, (c) cSN-GAN, and (d) rcSN-GAN with the estimated $T'$ on CIFAR-10 ((A) symmetric noise and (B) asymmetric noise). These models are discussed in Section~\ref{subsec:estT_eval}. In each picture block, each column shows samples associated with the same class. Each row includes samples generated from a fixed ${\bm z}$ and a varied $y^g$. The value below each picture block represents the achieved Intra FID (which is the same as the value reported in Tables~\ref{tab:estT_eval} and \ref{tab:estT_eval_ex}). When the score difference between the baseline models ((a) AC-CT-GAN and (c) cSN-GAN) and the proposed models ((b) rAC-CT-GAN with $T'$ and (d) rcSN-GAN with $T'$) is more than 3 points, we use bold font to indicate the better model. Refer to Figures~\ref{fig:ctgan_gen} and \ref{fig:sngan_gen} for comparison with rAC-CT-GAN with the known $T$ and rcSN-GAN with the known $T$, respectively.}  
  \label{fig:estT_gen}
\end{figure*}

\clearpage
\begin{table*}
  \centering
  \scalebox{0.8}{
    \begin{tabular}{c|c|cccc|cccc}
      \bhline{1pt}
      \multirow{2}{*}{Model}
      & \multirow{2}{*}{Metric}
      & \multicolumn{4}{c|}{CIFAR-10 (symmetric noise)}
      & \multicolumn{4}{c}{CIFAR-100 (symmetric noise)}
      \\ \cline{3-10}
      & & A & B & C & D
            & A & B & C & D
      \\ \bhline{0.75pt}
      & \multirow{2}{*}{FID $\downarrow$}
      & 27.9 & {\bf 14.7} & 12.4 & 13.5
            & 33.1 & {\bf 20.4} & 17.2 & 18.4
      \\
      &
      & (28.7) & (38.8) & (12.7) & (13.5)
            & (34.2) & (41.1) & (16.6) & (18.6)
      \\ \cline{2-10}      
      & \multirow{2}{*}{Intra FID $\downarrow$}
      & {\bf 55.7} & {\bf 34.6} & 33.4 & {\bf 36.9}
            & -- & -- & -- & --
      \\
      Improved rAC-GAN
      &
      & (58.9) & (92.4) & (34.6) & (41.3)
            & -- & -- & -- & --
      \\ \cline{2-10}
      (rAC-GAN)
      & \multirow{2}{*}{GAN-test $\uparrow$}
      & 65.1 & {\bf 77.7} & 78.2 & {\bf 63.5}
            & 26.2 & {\bf 22.5} & 21.5 & {\bf 15.4}
      \\
      &
      & (62.7) & (27.1) & (77.3) & (59.6)
            & (27.2) & (1.0) & (21.0) & (7.9)
      \\ \cline{2-10}
      & \multirow{2}{*}{GAN-train $\uparrow$}
      & 59.9 & {\bf 70.8} & 69.1 & {\bf 59.7}
            & 17.1 & {\bf 16.3} & 14.8 & {\bf 11.7}
      \\
      &
      & (58.7) & (26.3) & (67.6) & (51.9)
            & (17.4) & (1.0) & (14.1) & (6.9)
      \\ \hline
      & \multirow{2}{*}{FID $\downarrow$}
      & {\bf 30.4} & {\bf 16.9} & 14.2 & 14.9
            & \textit{50.2} & {\bf 25.8} & 18.0 & 18.7
      \\
      &
      & (35.4) & (22.2) & (14.8) & (14.8)
            & (40.1) & (31.2) & (17.5) & (19.6)
      \\ \cline{2-10}      
      & \multirow{2}{*}{Intra FID $\downarrow$}
      & {\bf 76.9} & {\bf 39.6} & {\bf 52.9} & {\bf 48.2}        
            & -- & -- & -- & --
      \\
      Improved rcGAN
      &
      & (89.2) & (51.5) & (63.2) & (76.8)
            & -- & -- & -- & --
      \\ \cline{2-10}
      (rcGAN)
      & \multirow{2}{*}{GAN-test $\uparrow$}
      & {\bf 27.3} & {\bf 65.7} & {\bf 38.9} & {\bf 48.8}
            & {\bf 4.5} & 12.0 & {\bf 9.5} & {\bf 6.1}
      \\
      &
      & (18.9) & (58.4) & (31.2) & (22.8)
            & (1.5) & (9.2) & (5.3) & (3.1)
      \\ \cline{2-10}
      & \multirow{2}{*}{GAN-train $\uparrow$}
      & {\bf 31.9} & {\bf 60.7} & {\bf 36.7} & {\bf 47.3}
            & {\bf 6.0} & 10.3 & 7.5 & 4.4
      \\
      &
      & (25.3) & (48.7) & (30.6) & (20.5)
            & (1.7) & (8.3) & (4.7) & (2.9)
      \\ \bhline{1pt}
    \end{tabular}
  }
  \vspace{2mm}
  \caption{Extended version of Table~\ref{tab:imp_eval}. Quantitative results using the improved technique. These results are discussed in Section~\ref{subsec:imp_eval}. In the second row, A, B, C, and D indicate DCGAN, WGAN-GP, CT-GAN, and SN-GAN, respectively. We evaluated the models in severely noisy settings (i.e., symmetric noise with a noise rate 0.9). Under the third row, each odd row contains the scores for the improved rAC-GAN or improved rcGAN and each even row (denoted in parenthesis) contains the scores for the naive rAC-GAN or naive rcGAN. Bold and italic fonts indicate that the score for the improved models is better or worse by more than 3 points than that for the naive models, respectively.}
  \label{tab:imp_eval_ex}
\end{table*}

\begin{figure*}[t]
  \centering
  \includegraphics[width=\textwidth]{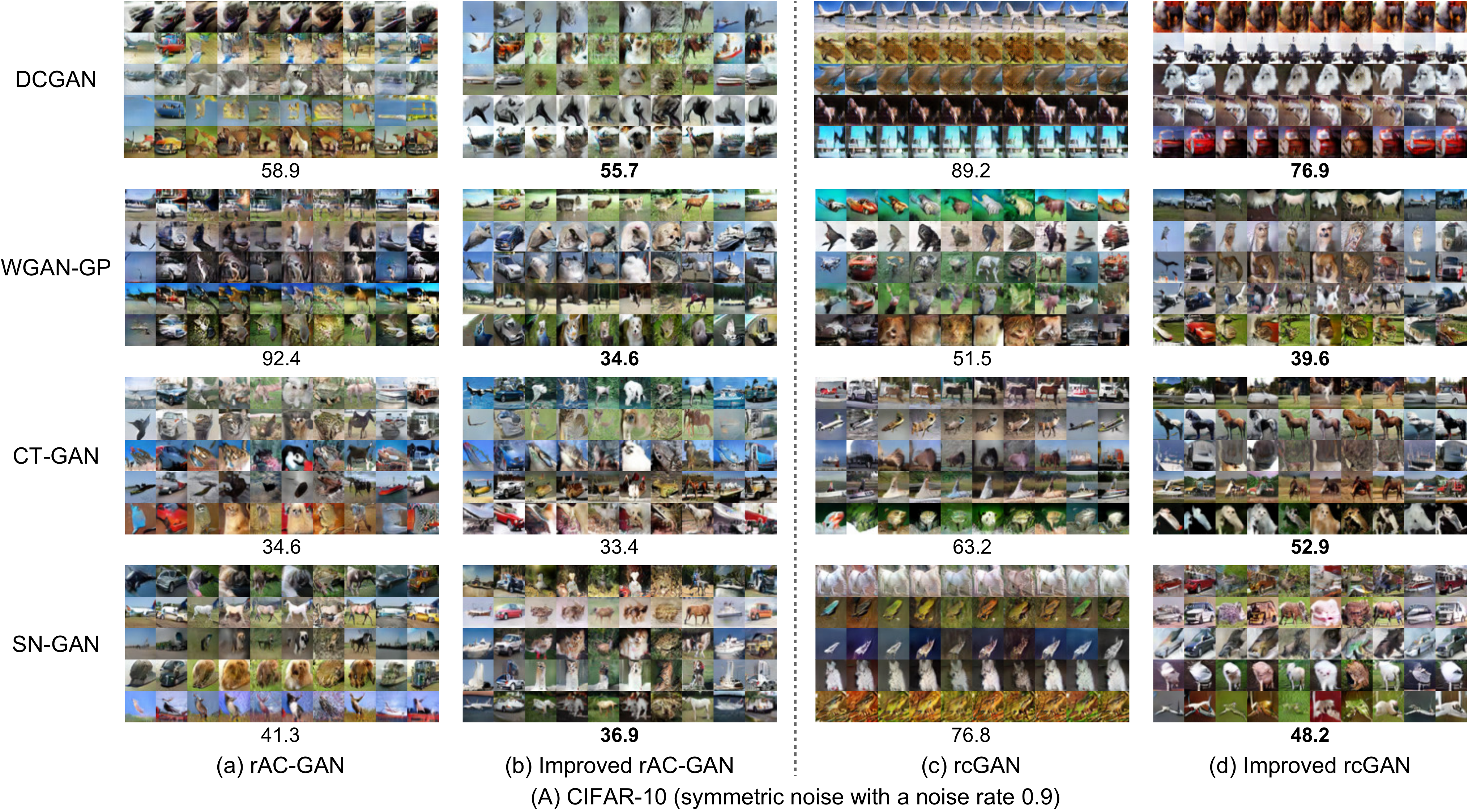}
  \caption{Image samples generated using (a) rAC-GAN, (b) improved rAC-GAN, (c) rcGAN, and (d) improved rcGAN on CIFAR-10 in severely noisy settings (i.e., symmetric noise with a noise rate 0.9). These models are discussed in Section~\ref{subsec:imp_eval}. In each picture block, each column shows samples associated with the same class. Each row includes samples generated from a fixed ${\bm z}$ and a varied $y^g$. The value below each picture block represents the achieved Intra FID (which is the same as the value reported in Tables~\ref{tab:imp_eval} and \ref{tab:imp_eval_ex}). When the score difference between the naive models ((a) rAC-GAN and (c) rcGAN) and the improved models ((b) improved rAC-GAN and (d) improved rcGAN) is more than 3 points, we use bold font to indicate the better model.}
  \label{fig:imp_gen}
\end{figure*}

\clearpage
\begin{figure*}[t]
  \centering
  \includegraphics[width=\textwidth]{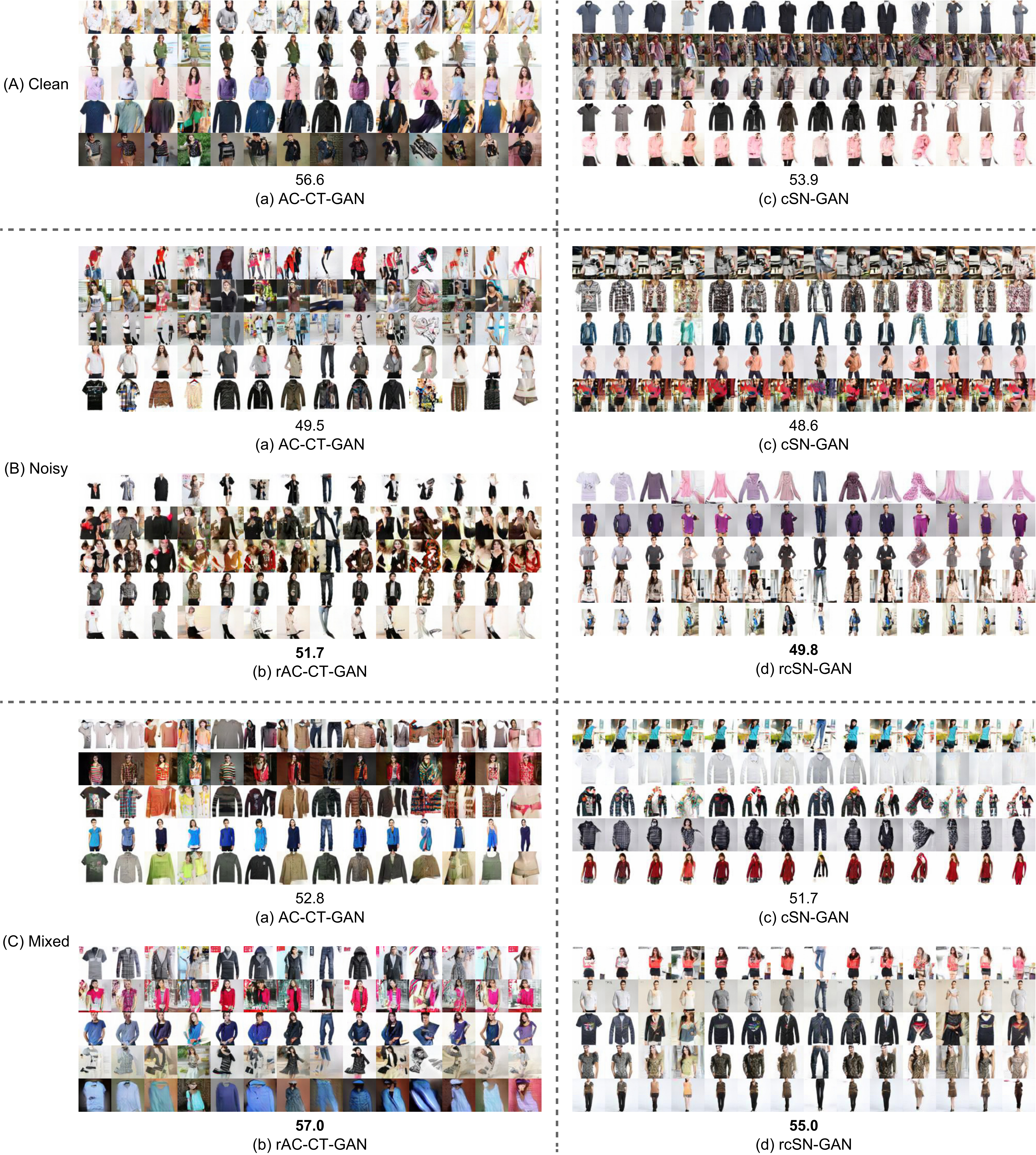}
  \caption{Image samples generated using (a) AC-CT-GAN, (b) rAC-CT-GAN, (c) cSN-GAN, and (d) rcSN-GAN on Clothing1M ((A) clean, (B) noisy, and (C) mixed settings). These models are discussed in Section~\ref{subsec:clothing1m_eval}. In each picture block, each column shows samples belonging to the same class. From left to right, each column represents t-shirt, shirt, knitwear, chiffon, sweater, hoodie, windbreaker, jacket, down coat, suit, shawl, dress, vest, and underwear, respectively. Each row includes samples generated from a fixed ${\bm z}$ and a varied $y^g$. The value below each picture block represents the achieved GAN-train (which is the same as the value reported in Table~\ref{tab:clothing1m_eval}). The larger the value, the better. Bold font indicates a better score in each block. Note that this dataset is challenging (annotation accuracy is only 61.54\%~\cite{TXiaoCVPR2015}) and correct labeling is also difficult for humans.}
  \label{fig:clothing1m_gen}
\end{figure*}

\clearpage
\section{Additional analysis}
\label{sec:ana}

\subsection{Effect of gap between real and model $T$}
\label{subsec:gap_ana}

In Section~\ref{subsec:estT_eval}, we evaluated the models with the estimated $T'$ and examined the effect when there is a gap between the real $T$ and model $T$ (particularly, $T'$ in this case). To further investigate such an effect, we conducted an additional experiment. In the following, to clarify the difference, we denote the real $T$ and model $T$ by $T^r$ and $T^g$, respectively, and their corresponding noise rates by $\mu^r$ and $\mu^g$, respectively. In Section~\ref{subsec:comp_eval}, we examined the performance change when $\mu^r$ and $\mu^g$ are varied at the same time (i.e., $\mu^r = \mu^g$). In contrast, in this section, to inspect the effect of the gap between $T^r$ and $T^g$, we fixed $\mu^r$ as a constant value ($\mu^r = 0$ or $\mu^r = 0.5$) and  investigated the performance change when $\mu^g$ is varied ($\mu^g \in \{ 0, 0.1, 0.3, 0.5, 0.7, 0.9 \}$). Figures~\ref{fig:gap0_eval} and \ref{fig:gap05_eval} show the results for noise rates $\mu^r = 0$ and $\mu^r = 0.5$, respectively. Although there is a dependency on the models, datasets, and evaluation metrics, we find that the quantity degradation is relatively small when the gap between $\mu^r$ and $\mu^g$ is within $\pm 0.2$. However, in this situation, the theoretical guarantees supported by Theorems~\ref{th:racgan} and \ref{th:rcgan} do not hold, and we admit that there is room to explore these observations theoretically in future work.

\subsection{Effect of learning rate}
\label{subsec:lr_ana}

Recent studies (e.g., \cite{DTanakaCVPR2018}) show that a high learning rate is useful for preventing a classifier DNN from memorizing noisy labels. To explore such an effect on conditional generative models, we performed a comparative study using the models with different learning rates. In particular, we evaluated the baseline models (i.e., AC-CT-GAN and cSN-GAN) and the proposed models (i.e., rAC-CT-GAN and rcSN-GAN) in severely noisy settings (i.e., symmetric noise with a noise rate 0.9). We selected the initial learning rate $\alpha$ from 0.0001, 0.0002, 0.0004, and 0.0008. As described in Appendix~\ref{sec:comp_eval_detail}, the default parameter of $\alpha$ is 0.0002. We trained the models for $100k$ generator iterations and decayed $\alpha$ to 0 over $100k$ iterations in all settings.

\smallskip\noindent\textbf{Results.}
We display the results in Figure~\ref{fig:lr_eval}. As discussed in the previous studies, generally the GAN training itself is not stable and has sensitivity to the learning rate (e.g., the authors of DCGAN~\cite{ARadfordICLR2016} recommended a low learning rate). Therefore, the relationship between the model and the learning rate in GANs  might be more difficult to explain than that in the classifier DNNs. However, we observed two tendencies through the experiments:
\begin{itemize}
  \setlength{\parskip}{1pt}
  \setlength{\itemsep}{1pt}
\item As the learning rate increases (particularly ranged from 0.0001 to 0.0004), the quantitative scores tend to become better in the proposed models; however, such benefits are small in the baseline models (particularly in cSN-GAN). We argue that this is because our proposed models can employ noisy labels as useful conditional information and this allows for suppressing the training instability resulting from a high learning rate.
\item However, when using an extensively high learning rate (e.g., 0.0008), the scores degrade even when using the proposed models (particularly when using rAC-CT-GAN). This implies the necessity of a careful parameter tuning.
\end{itemize}
As per the latest studies (e.g.,~\cite{IGulrajaniNIPS2017,TMiyatoICLR2018b}), the dependency on hyperparameter settings is being improved, and we expect that a more label-noise robust model will be constructed along with the advances in GANs.

\subsection{Effect of batch size}
\label{subsec:bs_ana}

Another important factor with regard to the training is the batch size. In particular, it might be critical in noisy label settings because, as the batch size becomes small, the factors for distinguishing between right and wrong labels also become fewer. To investigate this effect, we conducted a comparative study using the models with different batch sizes. We selected the batch sizes from 32, 64, and 128. As described in Appendix~\ref{sec:comp_eval_detail}, the default batch size is 64. In this analysis, we set the learning rate $\alpha$ to 0.0002 (default).

\smallskip\noindent\textbf{Results.}
We show the results in Figure~\ref{fig:bs_eval}. As was the case with the learning rate, the batch size affects the GAN training itself. Therefore, it is not easy to explain precisely the relationship between the model and the batch size. However, we observed a similar tendency to that of the learning rate, i.e., the proposed models benefit from an increasing batch size, whereas such benefits are small in the baseline models. The latest study~\cite{ABrockArXiv2018} demonstrates that, by incorporating some techniques, it is possible to obtain GAN training stability even when using a large batch size (e.g., a batch size of 2048). We expect that the performance of rAC-GAN and rcGAN will be improved along with such advances.

\subsection{Distance to noisy labeled data}
\label{subsec:intra_fid_noise_ana}

In the main text, we used Intra FID to measure the distance between the generated data distribution and the \textit{clean} labeled data distribution. Another interesting metric is the distance between the generated data distribution and the \textit{noisy} labeled data distribution. To assess it, we computed Intra FID between the samples generated by $G$ and the real samples belonging to the class of concern in terms of \textit{noisy} labels. We show the results in Figure~\ref{fig:intra_fid_noise}.\footnote{We calculated these scores only for CIFAR-10 with symmetric noise because in the other settings the number of noisy labeled data for each class is insufficient to use this metric.} These results support the finding, discussed in the last paragraph in Section~\ref{subsec:comp_eval}, i.e., cGAN can fit even noisy labels, and AC-GAN shows robustness for symmetric noise. We found that these tendencies occur independently of the GAN configurations.

\clearpage
\begin{figure*}[t]
  \centering
  \includegraphics[width=0.85\textwidth]{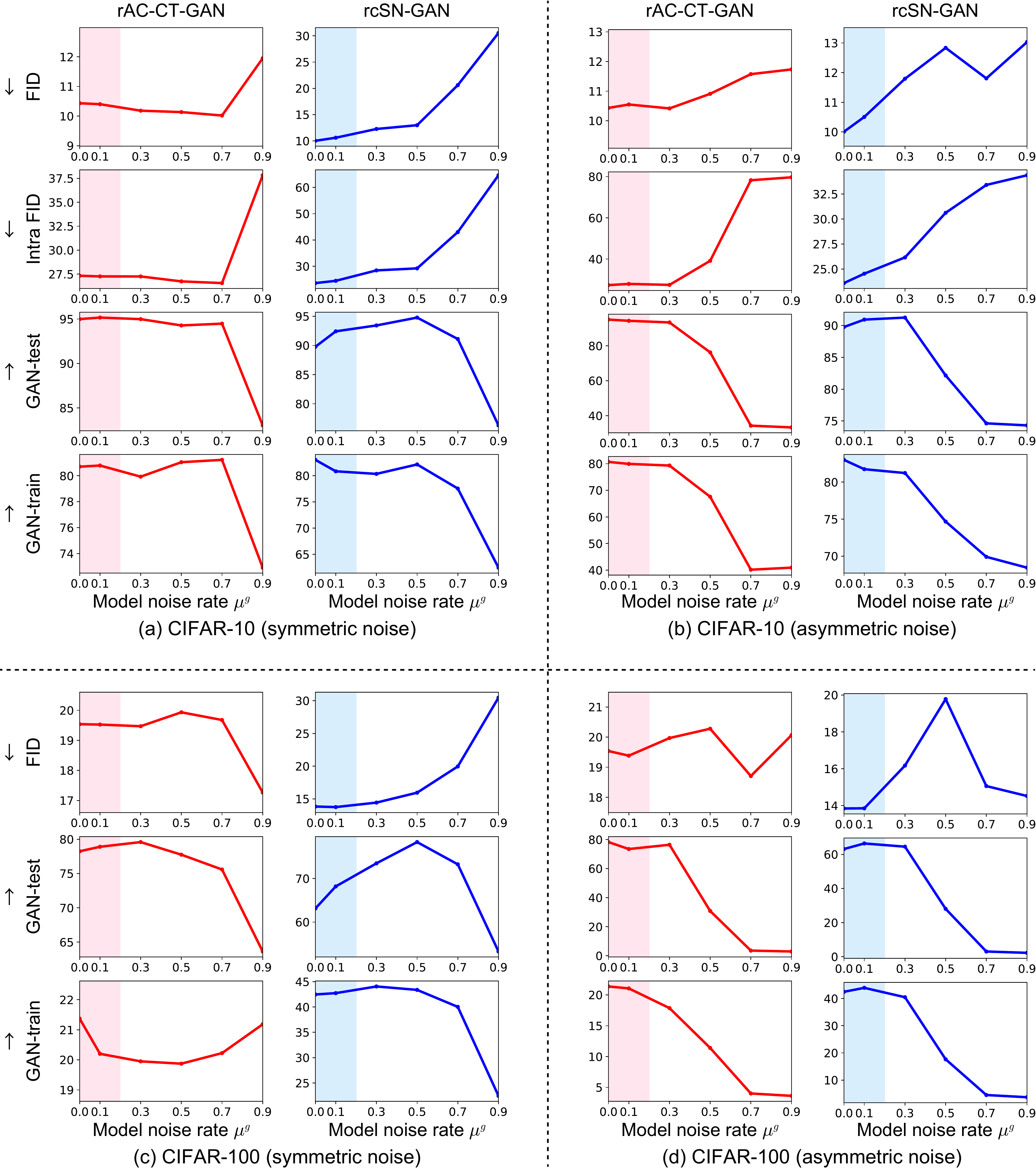}
  \caption{Effect of the gap between real and model $T$. We evaluated rAC-CT-GAN and rcSN-GAN on CIFAR-10 and CIFAR-100 in symmetric and asymmetric noise settings. We fixed a real noise rate as $\mu^r = 0$ and varied a model noise rate $\mu^g$ in $\{ 0, 0.1, 0.3, 0.5, 0.7, 0.9 \}$. The colored area indicates that the gap is within $\pm 0.2$. Note that the scale is adjusted on each graph for easy viewing.}
  \label{fig:gap0_eval}
\end{figure*}

\clearpage
\begin{figure*}[t]
  \centering
  \includegraphics[width=0.85\textwidth]{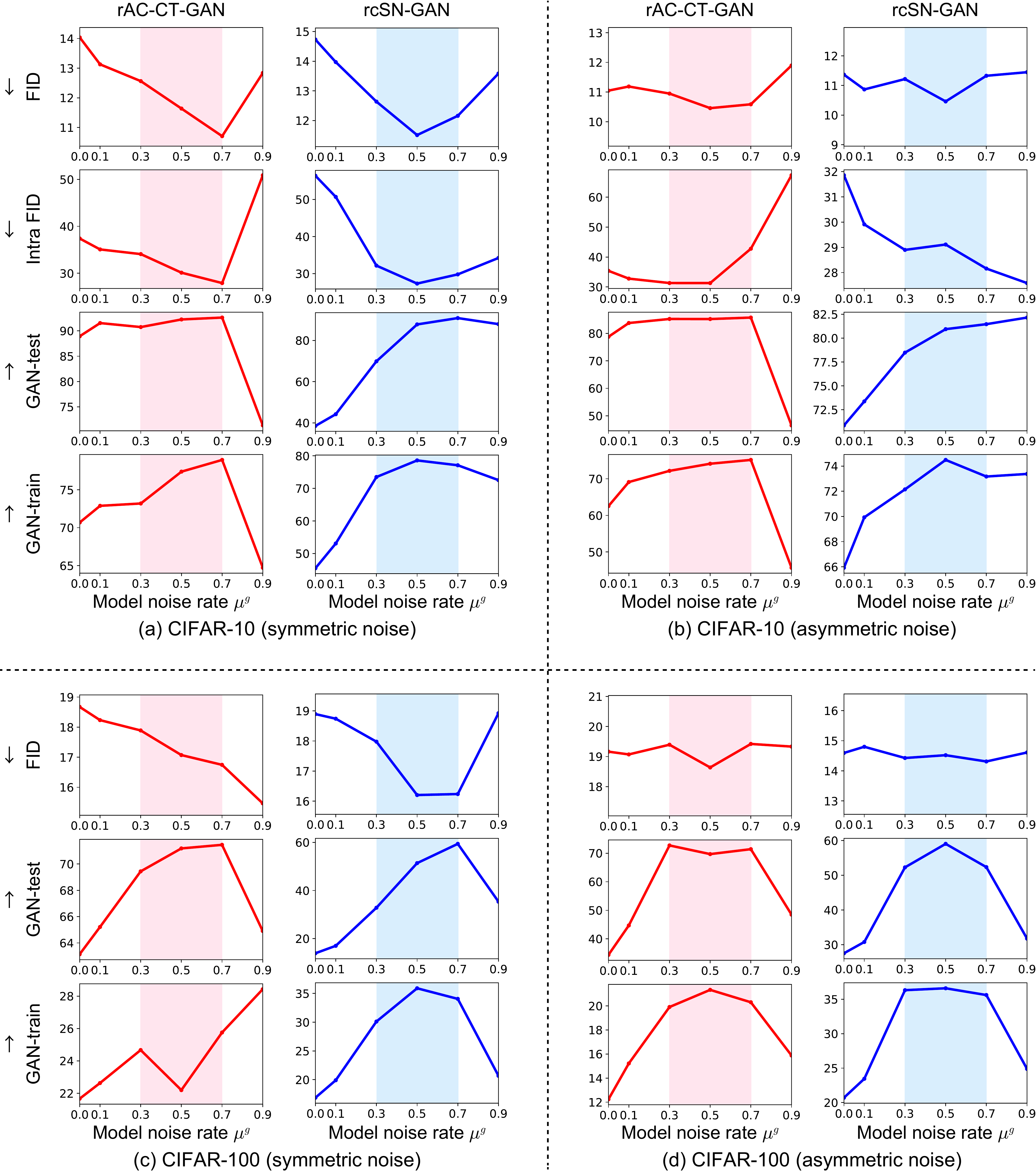}
  \caption{Effect of the gap between real and model $T$. We evaluated rAC-CT-GAN and rcSN-GAN on CIFAR-10 and CIFAR-100 in symmetric and asymmetric noise settings. We fixed a real noise rate as $\mu^r = 0.5$ and varied a model noise rate $\mu^g$ in $\{ 0, 0.1, 0.3, 0.5, 0.7, 0.9 \}$. The colored area indicates that the gap is within $\pm 0.2$. Note that the scale is adjusted on each graph for easy viewing.}
  \label{fig:gap05_eval}
\end{figure*}

\clearpage
\begin{figure}[t]
  \centering
  \includegraphics[width=0.9\columnwidth]{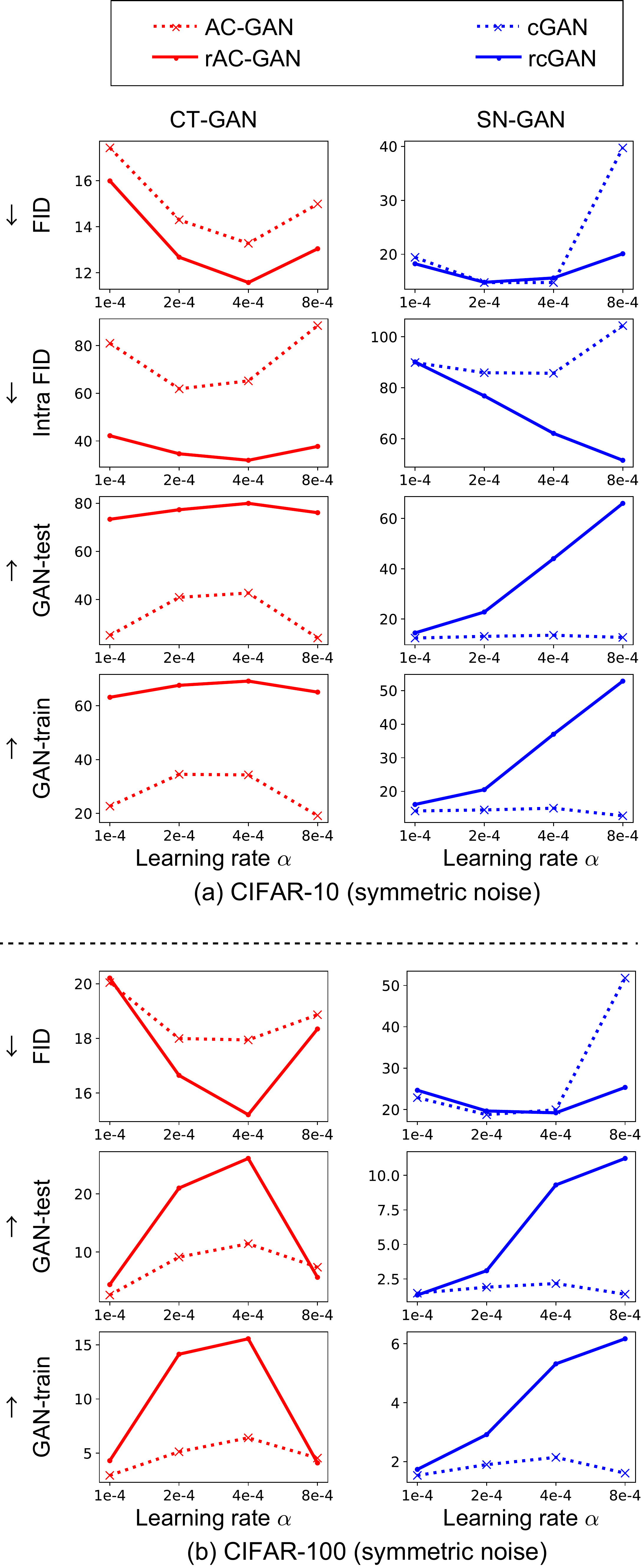}
  \caption{Comparison of the models with learning rates of 0.0001, 0.0002 (default), 0.0004, and 0.0008. We evaluated AC-CT-GAN, rAC-CT-GAN, cSN-GAN, and rcSN-GAN on CIFAR-10 and CIFAR-100 in severely noisy settings (i.e., symmetric noise with a noise rate 0.9). We fixed the batch size as 64 (default). Note that the scale is adjusted on each graph for easy viewing.}
  \label{fig:lr_eval}
\end{figure}

\begin{figure}[t]
  \centering
  \includegraphics[width=0.9\columnwidth]{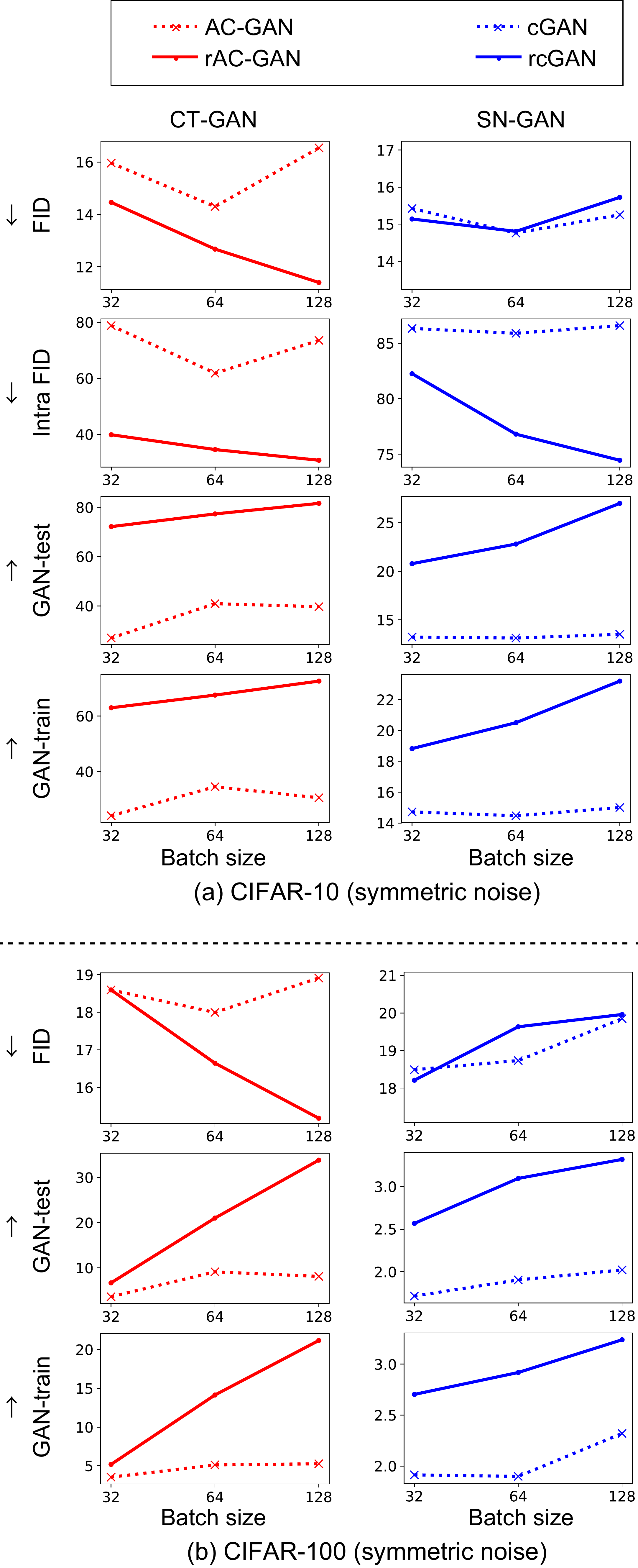}
  \caption{Comparison of the models with batch sizes of 32, 64 (default), and 128. We evaluated AC-CT-GAN, rAC-CT-GAN, cSN-GAN, and rcSN-GAN on CIFAR-10 and CIFAR-100 in severely noisy settings (i.e., symmetric noise with a noise rate 0.9). We fixed the learning rate $\alpha$ as 0.0002 (default). Note that the scale is adjusted on each graph for easy viewing.}
  \label{fig:bs_eval}
\end{figure}

\begin{figure*}[t]
  \centering
  \includegraphics[width=0.85\textwidth]{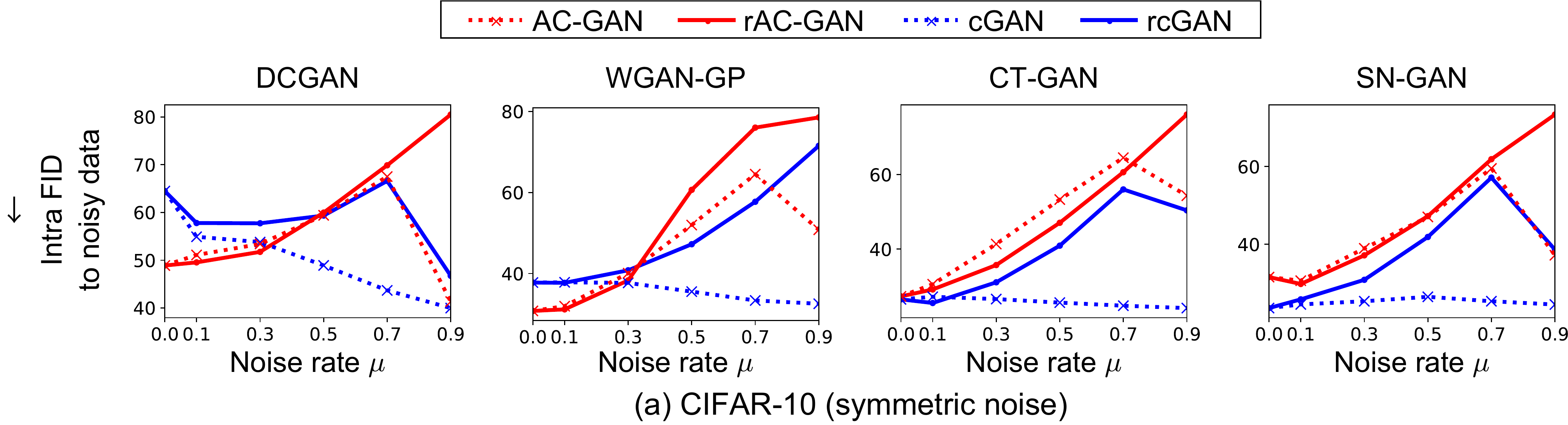}
  \caption{Intra FID between the real noisy data distribution and the generated data distribution. We evaluated AC-GAN, rAC-GAN, cGAN, and rcGAN with combinations of DCGAN, WGAN-GP, CT-GAN, and SN-GAN on CIFAR-10 with symmetric noise. Note that the scale is adjusted on each graph for easy viewing.}
  \label{fig:intra_fid_noise}
\end{figure*}

\clearpage
\section{Details on Section~\ref{subsec:comp_eval}}
\label{sec:comp_eval_detail}

\subsection{Network architectures and training settings}
\label{subsec:cifar_net}

In the experiments on CIFAR-10 and CIFAR-100 (Section~\ref{subsec:comp_eval}--\ref{subsec:imp_eval}), we tested four GAN configurations: DCGAN~\cite{ARadfordICLR2016}, WGAN-GP~\cite{IGulrajaniNIPS2017}, CT-GAN~\cite{XWeiICLR2018}, and SN-GAN~\cite{TMiyatoICLR2018b}. As discussed in Section~\ref{subsec:comp_eval}, instead of extensively searching for the best parameters for each label-noise setting, we tested them with the default parameters that are commonly used in clean label settings, and investigated the label-noise effect for them. We explain each one below.

\medskip\noindent\textbf{Notation.}
In the description of network architectures, we use the following notation.
\begin{itemize}
  \setlength{\parskip}{1pt}
  \setlength{\itemsep}{1pt}
\item FC:
  Fully connected layer
\item Conv:
  Convolutional layer
\item Deconv:
  Deconvolutional (i.e., fractionally strided convolutional) layer
\item BN:
  Batch normalization~\cite{SIoffeICML2015}
\item ReLU:
  Rectified unit~\cite{VNairICML2010}
\item LReLU:
  Leaky rectified unit~\cite{AMaasICML2013,BXuICMLW2015}
\item ResBlock:
  Residual block~\cite{KHeCVPR2016}
\item Concat($y$):
  Concatenating $y$ ($\in \{ 1, \dots, c \}$) after converting it to a one-hot vector ($\in \mathbb{R}^{c}$) and reshaping it to adjust feature size
\item Proj(Embed($y$)):
  Embedding $y$ such that its dimension becomes the same as of the previous layer ${\bm h}$ and taking an inner product between embedded $y$ and ${\bm h}$
\end{itemize}

In the description of training settings, we use the following notation. Note that we used the Adam optimizer~\cite{DPKingmaICLR2015} for all GAN training.
\begin{itemize}
  \setlength{\parskip}{1pt}
  \setlength{\itemsep}{1pt}
\item $\alpha$:
  Learning rate of Adam
\item $\beta_{1}$:
  The first order momentum parameter of Adam
\item $\beta_{2}$:
  The second order momentum parameter of Adam
\item $n_{D}$:
  The number of updates of $D$ per one update of $G$
\end{itemize}

\subsubsection{DCGAN}
\label{subsubsec:dcgan_detail}

DCGAN~\cite{ARadfordICLR2016} is a commonly used baseline model. The main principle of DCGAN is to compose the generator and discriminator using only convolutional layers along with batch normalization~\cite{SIoffeICML2015}. It shows promising results in image generation and unsupervised representation learning.

\medskip\noindent\textbf{Network architectures.}
We implemented standard CNN network architectures while referring to~\cite{TMiyatoICLR2018b,AOdenaICML2017}. We describe their details in Table~\ref{tab:net_cnn}. The conditional generators used in AC-GAN/rAC-GAN and cGAN/rcGAN are the same (Table~\ref{tab:net_cnn}(a)), while the discriminators are different. For AC-GAN/rAC-GAN, we used $D\mbox{/}C$ in which the layers are shared between $D$ and $C$ except for the last layer, following~\cite{AOdenaICML2017} (Table~\ref{tab:net_cnn}(b)).
For cGAN/rcGAN, we used the \textit{concat} discriminator~\cite{MMirzaArXiv2014} that employs the conditional information by concatenating the conditional vector to the feature vectors (Table~\ref{tab:net_cnn}(c)).

\medskip\noindent\textbf{Training settings.}
In DCGAN, a non-saturating loss~\cite{IGoodfellowNIPS2014} is used as a GAN objective function. We trained the networks for $100k$ iterations using Adam with $\alpha = 0.0002$, $\beta_{1} = 0.5$, $\beta_{2} = 0.999$, $n_{D} = 1$, and batch size of 64. In AC-GAN and rAC-GAN, we set the trade-off parameters $\lambda_{\rm AC}^r$ and $\lambda_{\rm AC}^g$ to 1.

\begin{table}[thb]
  \centering
  \scalebox{0.9}{
    \begin{tabular}{c}
      \bhline{1pt}
      (a) {\bf Conditional generator} $G({\bm z}, y)$
      \\ \bhline{0.75pt}
      ${\bm z} \in \mathbb{R}^{128} \sim {\cal N}(0, I)$,
      Concat($y$)
      \\ \hline
      FC $\rightarrow$ $4 \times 4 \times 512$, BN, ReLU
      \\ \hline
      $4 \times 4$, stride=2 Deconv 256, BN, ReLU
      \\ \hline
      $4 \times 4$, stride=2 Deconv 128, BN, ReLU
      \\ \hline
      $4 \times 4$, stride=2 Deconv 64, BN, ReLU
      \\ \hline
      $3 \times 3$, stride=1 Conv 3, Tanh      
      \\ \bhline{1pt}
      \\ \bhline{1pt}
      (b) {\bf AC-GAN/rAC-GAN discriminator} $D({\bm x})\mbox{/}C({\bm x})$
      \\ \bhline{0.75pt}
      RGB image ${\bm x} \in \mathbb{R}^{32 \times 32 \times 3}$
      \\ \hline
      $3 \times 3$, stride=1 Conv 64, BN, LReLU
      \\
      $4 \times 4$, stride=2 Conv 64, BN, LReLU
      \\ \hline
      $3 \times 3$, stride=1 Conv 128, BN, LReLU
      \\
      $4 \times 4$, stride=2 Conv 128, BN, LReLU
      \\ \hline
      $3 \times 3$, stride=1 Conv 256, BN, LReLU
      \\
      $4 \times 4$, stride=2 Conv 256, BN, LReLU
      \\ \hline
      $3 \times 3$, stride=1 Conv 512, BN, LReLU
      \\ \hline
      FC $\rightarrow$ 1 for $D$, FC $\rightarrow c$ for $C$
      \\ \bhline{1pt}
      \\ \bhline{1pt}
      (c) {\bf cGAN/rcGAN discriminator} $D({\bm x}, y)$
      \\ \bhline{0.75pt}
      RGB image ${\bm x} \in \mathbb{R}^{32 \times 32 \times 3}$,
      Concat($y$)
      \\ \hline
      $3 \times 3$, stride=1 Conv 64, BN, LReLU, Concat($y$)
      \\
      $4 \times 4$, stride=2 Conv 64, BN, LReLU, Concat($y$)
      \\ \hline
      $3 \times 3$, stride=1 Conv 128, BN, LReLU, Concat($y$)
      \\
      $4 \times 4$, stride=2 Conv 128, BN, LReLU, Concat($y$)
      \\ \hline
      $3 \times 3$, stride=1 Conv 256, BN, LReLU, Concat($y$)
      \\
      $4 \times 4$, stride=2 Conv 256, BN, LReLU, Concat($y$)
      \\ \hline
      $3 \times 3$, stride=1 Conv 512, BN, LReLU, Concat($y$)
      \\ \hline
      FC $\rightarrow$ 1
      \\ \bhline{1pt}
    \end{tabular}
  }
  \vspace{2mm}
  \caption{Standard CNN architectures for CIFAR-10 and CIFAR-100. The basic architectures are the same as those in~\cite{TMiyatoICLR2018b}. The slopes of all LReLU are set to 0.1. Following the AC-GAN paper~\cite{AOdenaICML2017}, in $D$ we adopted dropout (with a drop rate 0.5) after all convolutional layers.}
  \label{tab:net_cnn}
\end{table}

\subsubsection{WGAN-GP}
\label{subsubsec:wgangp_detail}

WGAN-GP~\cite{IGulrajaniNIPS2017} is one of the most widely-accepted models in the literature at present. It is an improved variant of WGAN~\cite{MArjovskyICML2017} and incorporates a gradient penalty (GP) term as an alternative to weight clipping. By using GP, it imposes a Lipschitz constraint on the discriminator (called the \textit{critic} in that work). This allows for stabilizing the training of a wide variety of GAN architectures without relying on heavy hyperparameter tuning. We defined the network architectures and training settings based on the source code provided by the authors of WGAN-GP.\footnote{\url{https://github.com/igul222/improved_wgan_training}}

\medskip\noindent\textbf{Network architectures.}
We used ResNet architectures provided in the WGAN-GP paper~\cite{IGulrajaniNIPS2017}. We describe their details in Table~\ref{tab:net_resnet}. As in DCGAN, the conditional generators used in AC-GAN/rAC-GAN and cGAN/rcGAN are the same (Table~\ref{tab:net_resnet}(a)), while the discriminators are different. For AC-GAN/rAC-GAN, we used $D\mbox{/}C$ in which the layers are shared between $D$ and $C$ except for the last layer, following~\cite{IGulrajaniNIPS2017} (Table~\ref{tab:net_resnet}(b)). For cGAN/rcGAN, we used the \textit{projection} discriminator~\cite{TMiyatoICLR2018} that incorporates the conditional information in a projection based manner (Table~\ref{tab:net_resnet}(c)).

\medskip\noindent\textbf{Training settings.}
In WGAN-GP, Wasserstein loss and GP are used as a GAN objective function. We set the trade-off parameter between them ($\lambda_{\rm GP}$) to 10. We trained the networks for $100k$ generator iterations using Adam with $\alpha = 0.0002$ (linearly decayed to 0 over $100k$ iterations), $\beta_{1} = 0$, $\beta_{2} = 0.9$, $n_{D} = 5$, and batch size of 64. In AC-GAN and rAC-GAN, we set the trade-off parameters $\lambda_{\rm AC}^r$ and $\lambda_{\rm AC}^g$ to 1 and 0.1, respectively.

\begin{table}[thb]
  \centering
  \scalebox{0.9}{
    \begin{tabular}{c}
      \bhline{1pt}
      (a) {\bf Conditional generator} $G({\bm z}, y)$
      \\ \bhline{0.75pt}
      ${\bm z} \in \mathbb{R}^{128} \sim {\cal N}(0, I)$
      \\ \hline
      FC $\rightarrow$ $4 \times 4 \times ch$
      \\ \hline
      ResBlock up $ch$
      \\ \hline
      ResBlock up $ch$
      \\ \hline
      ResBlock up $ch$
      \\ \hline
      BN, ReLU
      \\ \hline
      $3 \times 3$, stride=1 Conv 3, Tanh
      \\ \bhline{1pt}
      \\ \bhline{1pt}
      (b) {\bf AC-GAN/rAC-GAN discriminator} $D({\bm x})\mbox{/}C({\bm x})$
      \\ \bhline{0.75pt}
      RGB image ${\bm x} \in \mathbb{R}^{32 \times 32 \times 3}$
      \\ \hline
      ResBlock down 128
      \\ \hline
      ResBlock down 128
      \\ \hline
      ResBlock 128
      \\ \hline
      ResBlock 128
      \\ \hline
      ReLU
      \\ \hline
      Global pooling
      \\ \hline
      FC $\rightarrow$ 1 for $D$, FC $\rightarrow c$ for $C$
      \\ \bhline{1pt}
      \\ \bhline{1pt}
      (c) {\bf cGAN/rcGAN discriminator} $D({\bm x}, y)$
      \\ \bhline{0.75pt}
      RGB image ${\bm x} \in \mathbb{R}^{32 \times 32 \times 3}$
      \\ \hline
      ResBlock down 128
      \\ \hline
      ResBlock down 128
      \\ \hline
      ResBlock 128
      \\ \hline
      ResBlock 128
      \\ \hline
      ReLU
      \\ \hline
      Global pooling
      \\ \hline
      (FC $\rightarrow$ 1) + Proj(Embed($y$))
      \\ \bhline{1pt}
    \end{tabular}
  }
  \vspace{2mm}
  \caption{ResNet architectures for CIFAR-10 and CIFAR-100. The basic network architectures are the same as those in~\cite{IGulrajaniNIPS2017,XWeiICLR2018,TMiyatoICLR2018b}. In $G$'s ResBlock, conditional batch normalization~\cite{VDumoulinICLR2017b,HdVriesNIPS2017} was used to impose a conditional constraint on $G$. Following~\cite{IGulrajaniNIPS2017,XWeiICLR2018}, in WGAN-GP and CT-GAN, we set $ch = 128$ in $G$ and used global mean pooling in $D$. In CT-GAN, we applied dropout (with drop rates of 0.2, 0.5, and 0.5 from the upper block) after the second to fourth ResBlocks in $D$. Following~\cite{TMiyatoICLR2018b}, in SN-GAN, we set $ch = 256$ in $G$, used global sum pooling in $D$, and applied spectral normalization to all the layers in $D$.}
  \label{tab:net_resnet}
\end{table}

\begin{table*}[thb]
  \centering
  \scalebox{0.9}{
    \begin{tabular}{c|c|c|c|c|c}
      \bhline{1pt}
      GAN & Architecture & $G$ $ch$ & $D$ dropout & $D$ pooling & Objective function
      \\ \bhline{0.75pt}
      DCGAN & CNN & - & \ding{51} & - & Non-saturating loss
      \\ \hline
      WGAN-GP & ResNet & 128 & \ding{55} & Mean & Wasserstein loss + gradient penalty (GP)
      \\ \hline
      CT-GAN & ResNet & 128 & \ding{51} & Mean & Wasserstein loss + GP + consistency term (CT)
      \\ \hline
      SN-GAN & ResNet & 256 & \ding{55} & Sum & Hinge loss (with spectral normalization (SN))
      \\ \bhline{1pt}
    \end{tabular}
  }
  \vspace{2mm}
  \caption{Comparison of network architectures and training settings}
  \label{tab:net_diff}
\end{table*}

\subsubsection{CT-GAN}
\label{subsubsec:ctgan_detail}

At present, CT-GAN~\cite{XWeiICLR2018} is one of the state-of-the-art models. It is an improved variant of WGAN-GP and adds a consistency term (CT) to impose a Lipschitz constraint for the whole input domain. This contributes a further improvement in stabilizing the training and raises the quality of generated images. We reimplemented the model while referring to the source code provided by the authors of CT-GAN.\footnote{\url{https://github.com/biuyq/CT-GAN}}

\medskip\noindent\textbf{Network architectures.}
The network architectures of CT-GAN are the same as those of WGAN-GP except that dropout is used in $D$. We describe its detailed settings in the caption of Table~\ref{tab:net_resnet}.

\medskip\noindent\textbf{Training settings.}
The training settings of CT-GAN are the same as those of WGAN-GP except that CT is added in case CT-GAN. We set the trade-off parameter between the Wasserstein loss and CT ($\lambda_{\rm CT}$) to 2. The other settings (namely, $\alpha$, $\beta_{1}$, $\beta_{2}$, $n_{D}$, $\lambda_{\rm GP}$, batch size, and number of iterations as well as $\lambda_{\rm AC}^r$ and $\lambda_{\rm AC}^g$ in AC-GAN and rAC-GAN) are the same as those of WGAN-GP.

\subsubsection{SN-GAN}
\label{subsubsec:sngan_detail}

SN-GAN~\cite{TMiyatoICLR2018b} is also one of the state-of-the-art models at present. It introduces spectral normalization to impose a Lipschitz constraint. This helps stabilizing the discriminator training, and SN-GAN brings a breakthrough in image generation in complex settings (e.g., high-resolution image generation in ImageNet). We reimplemented the model based on the source code provided by the authors of SN-GAN.\footnote{\url{https://github.com/pfnet-research/sngan_projection}}

\medskip\noindent\textbf{Network architectures.}
The network architectures of SN-GAN are the same as those of WGAN-GP except that the feature maps are doubled in $G$, global sum pooling is used instead of global mean pooling in $D$, and spectral normalization is applied to all the layers in $D$. We describe the details in Table~\ref{tab:net_resnet}.

\medskip\noindent\textbf{Training settings.}
The training settings of SN-GAN are also nearly identical to those of WGAN-GP except that a hinge-loss~\cite{JHLimArXiv2017,DTRanArXiv2017} is used instead of Wasserstein loss and GP. The other settings including $\alpha$, $\beta_{1}$, $\beta_{2}$, $n_{D}$, batch size, and number of iterations as well as $\lambda_{\rm AC}^r$ and $\lambda_{\rm AC}^g$ in AC-GAN and rAC-GAN, are the same as those of WGAN-GP.

\subsubsection{Summary}

We summarize the difference in network architectures and training settings between DCGAN, WGAN-GP, CT-GAN, and SN-GAN in Table~\ref{tab:net_diff}.

\subsection{Evaluation metrics}
\label{subsec:cifar_eval_detail}

As discussed in Section~\ref{subsec:comp_eval}, we used four metrics for a comprehensive analysis: (1) the Fr\'{e}chet inception distance (FID), (2) Intra FID, (3) the GAN-test, and (4) the GAN-train. In this appendix, we describe the detailed procedure for calculating the scores.

\subsubsection{FID}
\label{subsubsec:fid_detail}

The FID~\cite{MHeuselNIPS2017} measures the 2-Wasserstein distance between $p^r$ and $p^g$, and is defined as
\begin{flalign}
  \label{eqn:fid}
  F(p^r, p^g) = & \, \| {\bm m}^r - {\bm m}^g \|^2_2
  \\ \nonumber
  + & \, {\rm Tr}({\bm C}^r + {\bm C}^g - 2({\bm C}^r {\bm C}^g)^{1/2}),
\end{flalign}
where $\{ {\bm m}^r, {\bm C}^r\}$ and $\{ {\bm m}^g, {\bm C}^g\}$ denote the mean and covariance of the final feature vectors of the Inception model calculated over real and generated samples, respectively. The authors show that the FID has correlation with human judgment and is more resilient to noise or mode-dropping than the Inception score~\cite{TSalimansNIPS2016} that is also commonly used in this field. We used the FID to assess the quality of an overall generative distribution. In the experiments, we computed the FID between the $50k$ samples generated by $G$ and all the samples in the training set. The implementation was based on the source code provided by the authors of FID.\footnote{\url{https://github.com/bioinf-jku/TTUR}}

Generally, in GANs, it is difficult to define the timing when to stop the training, partially because of the lack of an explicit likelihood measure. It is still an open issue, but as an approximate solution, we chose the best model (i.e., simulated early stopping) based on the FID, following~\cite{MLucicArXiv2017}. Precisely, we calculated the FID every $5k$ iterations and chose the best model in terms of the FID. We calculated the other scores (namely, Intra FID, the GAN-test, and the GAN-train) using this model and reported the results in this paper.

\subsubsection{Intra FID}
\label{subsubsec:intrafid_detail}

Intra FID~\cite{TMiyatoICLR2018} is a variant of the FID and calculates the FID for each class. The authors of Intra FID empirically observed that Intra FID had correlation with the diversity and visual quality in conditional image generation tasks. We used Intra FID to check the quality of a conditional generative distribution. In the experiments, we computed Intra FID between the $5k$ samples generated by $G$ and all the samples in the training set belonging to the class of concern. We reported the score averaged over the classes. We used this metric only for CIFAR-10 because in CIFAR-100, the number of clean labeled data for each class is insufficient to calculate Intra FID (which needs to be $\ge$2,048).

\subsubsection{GAN-test}
\label{subsubsec:gan_test_detail}

The GAN-test~\cite{KShmelkovECCV2018} is the accuracy of a classifier trained on real images and is evaluated on the generated images. This metric is developed for conditional generative models and approximates the precision (i.e., image quality) of them. As a classifier, we used PreAct ResNet-18 used in~\cite{HZhangICLR2018}, which is an 18-layer network with preactivation residual blocks~\cite{KHeECCV2016}. The implementation was based on the source code provided
by the authors of \cite{HZhangICLR2018}.\footnote{\url{https://github.com/facebookresearch/mixup-cifar10}} We used a cross-entropy loss as an objective function and trained 200 epochs with a batch size of 128. We set an initial learning rate to 0.1 and divided it by 10 after 100 and 150 epochs. Weight decay was set to 0.0001. The accuracy scores for the real test sets of CIFAR-10 and CIFAR-100 were 94.8\% and 75.9\%, respectively (which were the average scores over the last 10 epochs for three classifiers with random initializations). While calculating the GAN-test, we generated $50k$ samples for evaluation. We calculated the accuracy for them using the above three classifiers and reported their average scores.

\subsubsection{GAN-train}
\label{subsubsec:gan_trail_detail}

The GAN-train~\cite{KShmelkovECCV2018} is the accuracy of a classifier trained on generated images and evaluated on real images in a test set. This metric is also developed for conditional generative models and approximates the recall (i.e., diversity) of them. Regarding the classifier, we used the same network architecture and training settings as those used for the GAN-test (described in Appendix~\ref{subsubsec:gan_test_detail}). While training the classifier, we generated $50k$ samples as training samples. Using this classifier, we calculated the accuracy for the test set. We reported the scores averaged over the last 10 epochs.

\section{Details on Section~\ref{subsec:estT_eval}}
\label{subsec:estT_eval_detail}

In this appendix, we provide the details of the noise transition probability estimation method used in Section~\ref{subsec:estT_eval}. As discussed in Section~\ref{subsec:estT}, we used a \textit{robust two-stage training algorithm}~\cite{GPatriniCVPR2017}, which can estimate $T'$ independently of the main model (namely, an image classification model in \cite{GPatriniCVPR2017} and a conditional generative model in our case). In this algorithm, a noisy label classifier $C'(\tilde{y}|{\bm x})$ is first trained using noisy labeled data and then $T'$ is estimated via the following two steps:
\begin{flalign}
  \label{eqn:estimateT}
  \bar{\bm x}^i & = \argmax_{\bm x \in {\cal X}'} C'(\tilde{y} = i|{\bm x})
  \\
  \label{eqn:estimateT2}
  T_{i, j}' & = C'(\tilde{y} = j | \bar{\bm x}^i),
\end{flalign}
where ${\cal X}'$ is a dataset used for calculating $T'$. In practice, we used the training set as ${\cal X}'$. After estimating $T'$, the main model is trained using it.

While implementing the classifier, we used the network architecture and training settings that are similar to those used in calculating the GAN-test (see Appendix~\ref{subsubsec:gan_test_detail}). However, one possible problem encountered while solving Equations~\ref{eqn:estimateT} and \ref{eqn:estimateT2} using DNNs is the memorization effect~\cite{CZhangICLR2017}, i.e., $C'(\tilde{y}|{\bm x})$ can fit to noisy labels and make all probabilities to be zero or one. This causes difficulty in obtaining $T'$ having reasonable values. To alleviate the effect, we added an explicit regularization (i.e., added dropout with a drop rate 0.8 after the first convolutional layer in each residual block) to degrade training performance on noisy labeled data~\cite{DArpitICML2017}, conducted temperature scaling~\cite{CGuoICML2017} to mitigate the gap between accuracy and confidence, and took a $\alpha$-percentile in place of the $\argmax$ of Equation~\ref{eqn:estimateT}~\cite{AMenonICML2015,GPatriniCVPR2017} to eliminate the data strongly fitting the labels. Following~\cite{GPatriniCVPR2017}, we set $\alpha$ empirically. We used $\alpha = 97\%$ for the CIFAR-10 symmetric and asymmetric noise and set $\alpha = 100\%$ (i.e., $\argmax$ is directly used) and $\alpha = 99.7\%$ for the CIFAR-100 symmetric and asymmetric noise, respectively.

\section{Details on Section~\ref{subsec:imp_eval}}
\label{subsec:imp_eval_detail}

In this appendix, we present the implementation details of the improved technique for severely noise data, which is introduced in Section~\ref{subsec:imp} and is evaluated in Section~\ref{subsec:imp_eval}. Regarding the network architecture, we used shared networks between $D$ (or $D\mbox{/}C$ in AC-GAN/rAC-GAN) and $Q$. Following a sharing scheme between $D$ and $C$ in AC-GAN/rAC-GAN, we shared the layers between $D$ (or $D\mbox{/}C$) and $Q$ except for the last layer. The other parameters that we needed to define were the trade-off parameters between ${\cal L}_{\rm GAN}$ and ${\cal L}_{\rm MI}$, i.e., $\lambda_{\rm MI}^g$ and $\lambda_{\rm MI}^q$. We empirically defined the parameters dependent on the GAN configurations (i.e., DCGAN, WGAN-GP, CT-GAN, or SN-GAN) and model (i.e., rAC-GAN or rcGAN) but independent of the datasets (i.e., CIFAR-10 or CIFAR-100). We list them in Table~\ref{tab:imp_param}.

\begin{table}[thb]
  \centering
  \scalebox{0.9}{
    \begin{tabular}{c|c|c|c}
      \bhline{1pt}
      Model & GAN & $\lambda_{\rm MI}^q$ & $\lambda_{\rm MI}^g$
      \\ \bhline{0.75pt}
            & DCGAN & 0.01 & 0.04
      \\
      Improved & WGAN-GP & 1 & 0.02
      \\
      rAC-GAN & CT-GAN & 0.01 & 0.04
      \\
            & SN-GAN & 1 & 0.02
      \\ \hline
            & DCGAN & 1 & 0.04
      \\
      Improved & WGAN-GP & 1 & 0.04
      \\
      rcGAN & CT-GAN & 0.01 & 0.04
      \\
            & SN-GAN & 1 & 0.04
      \\ \bhline{1pt}
    \end{tabular}
  }
  \vspace{2mm}
  \caption{Hyperparameters for improved rAC-GAN and improved rcGAN}
  \label{tab:imp_param}
\end{table}

\section{Details on Section~\ref{subsec:clothing1m_eval}}
\label{sec:clothing1m_eval_detail}

\subsection{Network architectures and training settings}
\label{subsec:net_clothing1m}

In the experiments on Clothing1M (Section~\ref{subsec:clothing1m_eval}), we used two GAN configurations: CT-GAN for AC-GAN/rAC-GAN and SN-GAN for cGAN/rcGAN. We defined the network architectures and training settings while referring to the source code provided by the authors of SN-GAN~\cite{TMiyatoICLR2018b} (which is used for $64 \times 64$ dog and cat image generation).\footnote{\url{https://github.com/pfnet-research/sngan_projection}} The reason why we refer to this source code is that there is no previous study attempting to learn a generative model using Clothing1M, to the best of our knowledge. We experimentally confirm that its settings are reasonable for Clothing1M with no hyperparameter tuning.

\medskip\noindent\textbf{Network architectures.}
We describe the details on the network architectures in Table~\ref{tab:net_clothing1m}. They are basically similar to those in CIFAR-10 and CIFAR-100 (described in Appendix~\ref{subsec:cifar_net}) except that the input image size is different (that is $32 \times 32 \times 3$ in CIFAR-10 and CIFAR-100, while that is $64 \times 64 \times 3$ in Clothing1M) and feature map size is modified to adjust to the input size difference.

\medskip\noindent\textbf{Training settings.}
In CT-GAN, we set the trade-off parameters to $\lambda_{\rm GP} = 10$ and $\lambda_{\rm CT} = 2$, which are the same as those in CIFAR-10 and CIFAR-100. We trained the networks for $150k$ generator iterations using Adam with $\alpha = 0.0002$ (linearly decayed to 0 over the last $50k$ iterations), $\beta_{1} = 0$, $\beta_{2} = 0.9$, $n_{D} = 5$, and batch size of 64. We set the trade-off parameters $\lambda_{\rm AC}^r$ and $\lambda_{\rm AC}^g$ to 1 and 0.1, respectively. In SN-GAN, we used the same settings except that a GAN objective function is replaced from the Wasserstein loss + GP + CT to the hinge loss.

\begin{table}[thb]
  \centering
  \scalebox{0.9}{
    \begin{tabular}{c}
      \bhline{1pt}
      (a) {\bf Conditional generator} $G({\bm z}, y)$
      \\ \bhline{0.75pt}
      ${\bm z} \in \mathbb{R}^{128} \sim {\cal N}(0, I)$
      \\ \hline
      FC $\rightarrow$ $4 \times 4 \times 1024$
      \\ \hline
      ResBlock up $512$
      \\ \hline
      ResBlock up $256$
      \\ \hline
      ResBlock up $128$
      \\ \hline
      ResBlock up $64$
      \\ \hline
      BN, ReLU
      \\ \hline
      $3 \times 3$, stride=1 Conv 3, Tanh
      \\ \bhline{1pt}
      \\ \bhline{1pt}
      (b) {\bf AC-GAN/rAC-GAN discriminator} $D({\bm x})\mbox{/}C(y)$
      \\ \bhline{0.75pt}
      RGB image ${\bm x} \in \mathbb{R}^{64 \times 64 \times 3}$
      \\ \hline
      ResBlock down 64
      \\ \hline
      ResBlock down 128
      \\ \hline
      ResBlock down 256
      \\ \hline
      ResBlock down 512
      \\ \hline
      ResBlock down 1024
      \\ \hline
      ReLU
      \\ \hline
      Global mean pooling
      \\ \hline
      FC $\rightarrow$ 1 for $D$, FC $\rightarrow$ c for $C$
      \\ \bhline{1pt}
      \\ \bhline{1pt}
      (c) {\bf cGAN/rcGAN discriminator} $D({\bm x}, y)$
      \\ \bhline{0.75pt}
      RGB image ${\bm x} \in \mathbb{R}^{64 \times 64 \times 3}$
      \\ \hline
      ResBlock down 64
      \\ \hline
      ResBlock down 128
      \\ \hline
      ResBlock down 256
      \\ \hline
      ResBlock down 512
      \\ \hline
      ResBlock down 1024
      \\ \hline
      ReLU
      \\ \hline
      Global sum pooling
      \\ \hline
      (FC $\rightarrow$ 1) + Proj(Embed($y$))
      \\ \bhline{1pt}
    \end{tabular}
  }
  \vspace{2mm}
  \caption{ResNet architectures for Clothing1M. The basic network architectures are defined while referring to \cite{TMiyatoICLR2018b}. The detailed settings are similar to those described in Table~\ref{tab:net_resnet}. In $G$'s ResBlock conditional batch normalization~\cite{VDumoulinICLR2017b,HdVriesNIPS2017} was used to impose a conditional constraint on $G$. In CT-GAN, we used global mean pooling in $D$ and applied dropout (with drop rates of 0.2, 0.2, 0.5, and 0.5 from the upper block) after the second to fifth ResBlocks in $D$. In SN-GAN, we used global sum pooling in $D$ and applied spectral normalization to all the layers in $D$.
  }
  \label{tab:net_clothing1m}
\end{table}

\subsection{Evaluation metrics}
\label{subsec:clothing1m_eval_detail}

As discussed in Section~\ref{subsec:clothing1m_eval}, we used the FID and the GAN-train as evaluation metrics in these experiments. We did not use Intra FID because the number of clean labeled data for each class is insufficient to calculate Intra FID. We did not use the GAN-test because Clothing1M is a challenging dataset and we find that a trained classifier tends to be easily deceived by noisy labeled data.

The calculation procedure of the FID is the same as that for CIFAR-10 and CIFAR-100 (described in Appendix~\ref{subsubsec:fid_detail}). We calculated the FID between the $50k$ generated samples and all the samples in the training set (particularly we used $1M$ noisy data). The calculation procedure of the GAN-test is also similar to that for CIFAR-10 and CIFAR-100 (described in Appendix~\ref{subsubsec:gan_test_detail}). However, we modified the classifier network architecture and training settings so as to obtain the training stability when trained on real clean labeled data. Regarding the network architecture, we used the same network architecture as that for CIFAR-10 and CIFAR-100 except that dropout (with a drop rate 0.5) is used after the first convolutional layer in each residual block. With regards to the training settings, we used a cross-entropy loss as an objective function and trained 200 epochs with a batch size of 128. We set an initial learning rate to 0.01 and divided it by 10 after 100 and 150 epochs. Weight decay was set to 0.01. The accuracy for the real clean labeled test sets was 71.1\%.\footnote{This score cannot be directly compared with the scores in the previous studies because we used $64 \times 64$ images but the previous studies used $256 \times 256$ images.} While training the classifier, we generated $50k$ samples as training samples. Using this classifier, we calculated the accuracy for a test set. We reported the scores averaged over the last 10 epochs.

\end{document}